\documentclass[mnsc]{informs3aa} 
\usepackage{subfigure}
\usepackage{arydshln}
\usepackage{makecell}

\OneAndAHalfSpacedXI 

\usepackage{endnotes}

\usepackage{longtable}
\usepackage{tabularx}
\usepackage{booktabs}
\usepackage{etex}
\usepackage{multirow}
\usepackage{array}
\usepackage{bbm}

\usepackage[title]{appendix}

\usepackage{threeparttable}

\newcommand{\mO}{\mathcal O}

\newcommand{\vct}{\boldsymbol }

\newcommand{\op}{\mathrm{op}}

\renewcommand{\tilde}{\widetilde}
\renewcommand{\hat}{\widehat}
\renewcommand{\bar}{\overline}

\newcommand{\sd}{\mathfrak d}

\newcommand{\xs}{x_{\mathsf s}}

\renewcommand{\hat}{\widehat}
\renewcommand{\tilde}{\widetilde}

\usepackage{amsbsy, amsfonts, amsgen, amsmath, amsopn, amssymb, amstext,
amsxtra, bezier, color, enumerate, graphicx, latexsym, verbatim,
pictexwd, supertabular, url, dsfont,leftidx, mathrsfs, appendix, setspace}
\usepackage{algorithm}
\usepackage{algorithmicx}
\usepackage{algpseudocode}
\makeatletter
\def\BState{\State\hskip-\ALG@thistlm}
\makeatother

\usepackage{pifont}

\usepackage{natbib}
 \bibpunct[, ]{(}{)}{,}{a}{}{,}%
 \def\bibfont{\small}%
\usepackage[colorlinks=true,bookmarks=false,urlcolor=blue, citecolor=blue,linkcolor=blue,bookmarksopen=false,draft=false]{hyperref}

\newcommand{\mX}{\mathcal X}
\newcommand{\mZ}{\mathcal Z}

\usepackage{parskip} 

\newcommand{\tot}{\mathsf{tot}}
		
\TheoremsNumberedThrough     
\ECRepeatTheorems

\EquationsNumberedThrough    

\MANUSCRIPTNO{} 
\renewcommand{\hat}{\widehat}
\renewcommand{\tilde}{\widetilde}
\renewcommand{\bar}{\overline}

\definecolor{DSgray}{cmyk}{0,1,0,0}

\begin{document}


\RUNAUTHOR{Anonymous}

\RUNTITLE{Continuum-armed Bandit Optimization with Pairwise Comparisons}

\TITLE{Continuum-armed Bandit Optimization with Batch Pairwise Comparison Oracles}

\ARTICLEAUTHORS{%
\AUTHOR{Xiangyu Chang \thanks{Author names listed in alphabetical order.}}
\AFF{Department of Information Systems and Intelligence Business, Xi'an Jiaotong University, Xi'an 710049, China}
\AUTHOR{Xi Chen}
\AFF{Stern School of Business, New York University, New York, NY 10012, USA}
\AUTHOR{Yining Wang}
\AFF{Naveen Jindal School of Management, University of Texas at Dallas, Richardson, TX 75080, USA}
\AUTHOR{Zhiyi Zeng}
\AFF{Department of Information Systems and Intelligence Business, Xi'an Jiaotong University, Xi'an 710049, China}
} 

\ABSTRACT{This paper studies a bandit optimization problem where the goal is to maximize a function $f(x)$ over $T$ periods for some unknown strongly concave function $f$. We consider a new pairwise comparison oracle, where the decision-maker chooses a pair of actions $(x, x')$ for a consecutive number of periods and then obtains an estimate of $f(x)-f(x')$. We show that such a pairwise comparison oracle finds important applications to joint pricing and inventory replenishment problems and network revenue management. The challenge in this bandit optimization is twofold. First, the decision-maker not only needs to determine a pair of actions $(x, x')$ but also a stopping time $n$ (i.e., the number of queries based on $(x, x')$). Second, motivated by our inventory application, the estimate of the difference $f(x)-f(x')$ is biased, which is different from existing oracles in stochastic optimization literature. To address these challenges, we first introduce a discretization technique and local polynomial approximation to relate this problem to linear bandits. Then we developed a tournament successive elimination technique to localize the discretized cell and run an interactive batched version of LinUCB algorithm on cells. We establish regret bounds that are optimal up to poly-logarithmic factors.
Furthermore, we apply our proposed algorithm and analytical framework to the two operations management problems and obtain results that improve state-of-the-art 
results in the existing literature.

This version: \today
}



\date{}

\maketitle

\section{Introduction}

We consider an important question of noisy \emph{continuum-armed bandit} (CAB) optimization, where the goal is to maximize of an unknown $d$-dimensional
smooth function $f:\mX \to\mathbb R$ over $T$ consecutive time periods,
where $\mX\subseteq[0,1]^d$ is a certain convex compact domain.
In the standard noisy CAB setting, at each time period $t\in\{1,2,\cdots,T\}$ an optimization algorithm takes an action $x_t\in\mX$
and observes a noisy feedback of the function value $f(x_t)$.
In this paper, we study a weaker feedback oracle: the algorithm at each time period $t$ should offer two actions $x_t,x_t'\in\mX$,
and a noisy evaluation of the \emph{difference} of the function values $f(x_t')-f(x_t)$ will be provided.
In this way, the algorithm no longer has noisy evaluation of function values directly, and can only perform optimization by (noisily) \emph{comparing}
two hypothetical solutions $x_t$ and $x_t'$ at each time period.
A formal definition of the pairwise comparison oracle is given in Sec.~\ref{sec:formulation} later in the paper.

The weaker pairwise comparison oracle is relevant in many business application scenarios in which the direct observation of feedback or rewards
is not possible or unbiased. Below we list several important applications of the CAB optimization problem when only pairwise comparison evaluations can be 
reliably made:
\begin{enumerate}
\item \textbf{Joint pricing and inventory replenishment with censored demand and lost sales.} 
Consider the standard news-vendor inventory control model with a pricing component: at the beginning of review period $t$
with an $x_t$ amount of carry-over inventory, the retailer makes a joint price decision $p_t$ and an inventory order-up-to level decision $y_t\geq x_t$.
The instantaneous reward $r_t$ is formulated as
\begin{equation}
r_t = \underbrace{p_t\max\{d_t,y_t\}}_{\text{sales revenue}} - \underbrace{c(y_t-x_t)}_{\text{ordering cost}} - \underbrace{h\max\{0,y_t-d_t\}}_{\text{holding cost}}
- \underbrace{b\max\{0, d_t-y_t\}}_{\text{lost sales cost}},
\label{eq:pricing-inventory-r}
\end{equation}
where $d_t$ is a realized demand depending on the price decision $p_t$.
Because the observed demand realization $\max\{d_t,y_t\}$ is \emph{censored} (since the maximum fulfilled demand during review period $t$ is at most $y_t$),
the retailer cannot directly observe the instantaneous reward $r_t$ (the lost-sales part is unobservable).
On the other hand, the recent work of \cite{chen2023optimal} shows that in the censored demand setting,
pairwise comparisons or expected rewards between two inventory decisions can still be reliably obtained.
In Sec.~\ref{subsec:om-inventory} we give a formal description of how the algorithms in this paper can be applied to this important supply chain management problem,
with extensions to multiple products too.

\item \textbf{Network revenue management with non-parametric demand learning.}
 Network Revenue Management (NRM) is a classical problem in operations management. In NRM, the retailer sells $\sd$ types of products over $T$ consecutive time periods,
by posting a $d$-dimensional vector $p_t$ at the beginning of each time period. The realized demands at time $t$, $d_t\in\mathbb R^\sd$, depend on the price being posted.
The sale of a unit of product type $i\in[\sd]$ also consumes certain amounts of resources, which have $\sd'$ types and the retailer has $\gamma_j T$ amount of non-replenishable inventory
for resource type $j$ initially. The objective for the retailer is to carry out price optimization
so that the expected cumulative revenue $\mathbb E[\sum_{t=1}^T\langle d_t,p_t\rangle]$ is maximized, while keeping track of resource inventory constraints
because once a resource type is depleted no product that uses that particular resource type could be solved.

The question of NRM, especially its learning-while-doing version in which the demand curve $\mathbb E[d_t|p_t]$ is \emph{not} known a priori and must be learnt on the fly
(which is nicknamed ``blind'' NRM in many cases), has been studied extensively in the literature \citep{besbes2012blind,chen2019nonparametric}.
However, when the demand curve is nonparametric, tight regret bounds especially in terms of dependency on the number of products/resource types are challenging;
see for example the summary in \citep{miao2021network} for a literature survey.

In full-information, offline optimization settings in which $\mathbb E[y_t|p_t]$ is known in advance, the conventional approach when the demand curve is invertible 
 is to optimize the problem with the expected demand rates $\lambda\in\mathbb R_{++}^\sd$ as optimization variables, leading to an optimization problem
with concave objectives and linear constraints.
Under the bandit setting, this approach can still be used as showcased in \citep{miao2021network}, with two additional complexities: 1) the inverse demand curve must be estimated,
which introduces bias into the observation of the objective function, and 2) the feasible region must be dealt with in an approximate manner, as the expected demand of a price vector
is not exactly known. In this paper, we use the pairwise comparison oracle as an abstraction that captures the first challenge and the aggregated $\psi$ penalty term in Assumption \ref{asmp:knapsack} and Eq.~(\ref{eq:defn-regret}) to capture the second challenge.
More technical details are given in Example \ref{exmp:nrm} and Sec.~\ref{subsec:om-nrm} to cover this example.

\end{enumerate}


\subsection{Our contributions}

In this paper, focusing on the a weaker pairwise comparison oracle that is intrinsically biased and only consistent when many time periods are involved, 
we make a number of methodological and theoretical contributions:
\begin{itemize}
\item When the objective function is very smooth but not necessarily concave, we propose an algorithm that achieves sub-linear regret. The regret scaling with time horizon $T$ also matches
existing lower bounds in the simpler setting where unbiased function value observations are available.
To achieve such results, our proposed algorithms contain several novel components such as successive tournament eliminations and local polynomial regression with biased pairwise comparisons.
Our approach improves on several existing results, including those that apply only to \emph{unbiased} function value observations or only moderately smooth objectives.
We discuss related works in details in the next section.

\item When the objective function is moderately smooth and strongly concave, we propose an algorithm to achieve $\sqrt{T}$ regret, which is a significant improvement compared with the case where
the objective function is merely smooth but without additional structures.
The proposed algorithm is based on the idea of gradient descent with fixed step sizes. While such an idea looks simple, its implementation must be very carefully designed because
only \emph{inexact} gradients resulting from biased pairwise comparison oracles are available.
In our algorithm, we use epochs with geometrically increasing number of time periods for each epoch to properly control the error in inexact gradient estimates,
so that when they are coupled together with fixed step sizes the optimal rate of convergence is attained.
\end{itemize}

In addition to the above contributions specific to the pairwise comparison oracle model, we also connect the oracle model with learning-while-doing problems
in important operations management questions, and show that our proposed methods and analysis have the potential to improve existing state-of-the-art results on these problems:

\begin{itemize}
\item We mention two important operations management problems: network revenue management and joint pricing and inventory replenishment. We explain how our proposed algorithms
could be applied to these problems with carefully designed pairwise comparison oracles, and demonstrate that in both cases, our algorithms and analysis improve existing results
in terms of dependency on time horizon $T$ when the objective is smoother, or several dimensional factors that are important in problems with many product or resource types.
\end{itemize}

\subsection{Related works}

The \emph{dueling bandit} is a related thread in bandit research that also features observational models involving pairwise comparisons of two arms.
\cite{yue2012k} pioneered this line of research, studying a bandit problem with finite number of arms under strong stochastic transitivity assumptions.
The works of \cite{urvoy2013generic,saha2022versatile,saha2021adversarial,dudik2015contextual} extended the dueling bandit model to pure-exploration, gap-dependent, adversarial
and/or contextual settings, retaining a finite number of arms/actions.
\cite{sui2018advancements} provides a survey of results in this direction.
The works of \cite{argarwal2022batched,agarwal2022asymptotically} studied \emph{batched} dueling bandit problems,
which remotely resemble our setting where \emph{many} time periods must be devoted to extract meaningful pairwise comparison signals.
\citep{argarwal2022batched,agarwal2022asymptotically} nevertheless adopt feedback models different from our biased pairwise comparison oracle,
and focus exclusively on multi-armed bandit settings.
It should be pointed out that very few existing works on dueling bandit allows for an infinite number of actions (see next paragraph for some exceptions), or adopts a \emph{flexible} budget of time periods devoted to each query,
with consistency only achieved over many time periods and a single time period being not informative (see e.g.~Definition \ref{defn:oracle}).

Our research is also connected with the literature on \emph{continuum-armed bandit}, involving an infinite (in fact, uncountable) number of arms
with non-parametric reward functions.
In the classical setting of continuum-armed bandit, every time period a single action is taken, and an unbiased reward observation is received.
 \cite{agrawal1995continuum} pioneered the research in this area, with extensive follow-up studies in \citep{bubeck2011x,locatelli2018adaptivity,wang2019optimization}
focusing on adaptivity and optimality. 
The work of \citep{kumagai2017regret} studies dueling bandit settings for continuum-armed bandit.
Compared to our settings, the work of \cite{kumagai2017regret} assumes the probability of seeing a binary output is closely related to the difference between function values at 
two actions being compared against, which is a special case of our comparison oracle and cannot be applied to our studied OM problems
that only deliver consistent comparisons with a large number of time periods.

Two important contributions this paper make are to apply the algorithmic framework (together with its analysis) to important operations management problems
to obtain improved results.
Here we review related works specific to the operations management applications studied in this paper, and compare them with results obtained in this paper:

\begin{table}[t]
\centering
\caption{Summary of results on joint pricing and inventory management with demand learning. Additional references are given in the main text.}
\label{tab:related-works-inventory}
\begin{flushleft}
{\footnotesize Note: references: HR09 \citep{huh2009nonparametric}; YLS21 \citep{yuan2019marrying}; CWZ23 \citep{chen2023optimal}; CSWZ22 \citep{chen2022dynamic}.
$\sd$ denotes the number of produces whose inventory and pricing are managed.
In all regret bounds, polynomial dependency on $\sd$ and other problem parameters other than time horizon $T$ is dropped. }\\
\end{flushleft}
\vskip 0.2in
\scalebox{0.9}{
\begin{tabular}{l|ccc|cc}
\hline
& pricing model& demand censoring& fixed costs& upper bound& lower bound\\
\hline
HR09& None& Yes& No& $\tilde O(\sqrt{T})$\textsuperscript{$*$}& N/A\\
YLS21& None& Yes& Yes& $\tilde O(\sqrt{T})$& N/A\\
CSWZ22& parametric& No& Yes& $\tilde O(\sqrt{T})$& N/A\\
CWZ23& non-parametric $\Sigma_\sd(2,M)$& Yes& No& $\tilde O(T^{\frac{\sd+2}{\sd+4}})$& $\Omega(T^{3/5})$\textsuperscript{$\dagger$}\\
\hline
\textbf{This paper}& non-parametric $\Sigma_\sd(k,M)$& Yes& No& $\tilde O(T^{\frac{k+\sd}{2k+\sd}})$&$ \Omega(T^{\frac{k+\sd}{2k+\sd}})$ \textsuperscript{$\sharp$}\\
\hline
\end{tabular}
}
\vskip 0.2in
\begin{flushleft}
{\footnotesize \textsuperscript{$*$}When demand distribution is equipped with a PDF that is uniformly bounded away from below, the regret upper bound could be improved to $O(\log T)$. Such improvement cannot happen when there is a pricing model with at least two parameters, as established in \citep{broder2012dynamic}.

\textsuperscript{$\dagger$}Applicable to $\sd=1$ only.

\textsuperscript{$\sharp$}Not proved in this paper but obtained by invoking existing lower bound results \citep{wang2019optimization}.}
\end{flushleft}
\end{table}

\begin{itemize}
\item \textbf{Joint pricing and inventory management with demand learning}. Table \ref{tab:related-works-inventory} summarizes the most related existing results on joint pricing and inventory management with demand learning. Compared with these results, application of the proposed algorithm and analytical frameworks in this paper
yields improved regret bounds for smoother pricing functions in $\Sigma_\sd(k,M)$ with $k>2$. (The case of $k=2$ would reduce to the $\tilde O(T^{(\sd+2)/(\sd+4)})$
regret upper bound obtained in \citep{chen2023optimal}.)

In addition to the works listed in Table \ref{tab:related-works-inventory}, other related works on joint pricing and inventory control with demand learning
include \citep{katehakis2020dynamic,chen2021nonparametric} focusing on specialized settings/algorithms such as discrete demands/inventory inflating policies.
We also remark that, as far as we know, there has been no works formally analyzing the problem when both censored demands and fixed ordering costs are present
(corresponding to two ``Yes'' in columns 3 and 4 of Table \ref{tab:related-works-inventory}), which would be an interesting future research direction to pursue.
More broadly, the problem originates from the studies of \cite{chen2004coordinating,chen2004infinite,HJ2008,PD1999,CS2012} focusing primarily on full-information settings.

\item \textbf{Blind network revenue management (NRM)}. Table \ref{tab:related-works-nrm} summarizes the most related existing results on blind network revenue management. 
Compared with these results, applications of our porposed algorithms and anlytical framework yields improved regret upper bounds.
More specifically, compared with the most related work of \cite{chen2023optimal} (CWZ23), our obtained regret bounds impose the same set of assumptions
and achieve the same asymptotic rate of $\tilde O(\sqrt{T})$, but has much improved dependency on the number of product/resource types,
improving a factor of $\sd^{3.5}$ to $\sd^{2.25}$ in the dominating regret terms.
Other results summarized in Table \ref{tab:related-works-inventory} either impose stronger assumptions on the demand function (parametric assumptions
or \emph{very} smooth non-parametric assumptions, with $k=\infty$), or achieve weaker regret upper bounds such as $\tilde O(T^{2/3})$ or $O(T^{1/2+\epsilon})$ for any $\epsilon>0$.

\begin{table}[t]
\centering
\caption{Summary of results on blind network revenue management. Additional references are given in the main text.}
\label{tab:related-works-nrm}
\begin{flushleft}
{\footnotesize Note: references: BZ12 \citep{besbes2012blind}; FSW18 \citep{ferreira2018online};  CJD19 \citep{chen2019nonparametric}; MWZ21 \citep{miao2021general}, MW21 \citep{miao2021network}.
$\sd$ denotes the number of product types, which is also the dimension of pricing vectors offered to customers at each time period.
In all regret bounds, polynomial dependency on problem parameters other than time horizon $T$ or $\sd$ is dropped. 
$f$ refers to the expected revenue as a function of demand rates.}\\
\end{flushleft}
\vskip 0.2in
\scalebox{1.0}{
\begin{tabular}{l|l|cc}
\hline
& demand model assumptions& upper bound& lower bound\\
\hline
BZ12& non-parametric, $\sd=1$, $f\in\Sigma_1(1,M)$& $\tilde O(T^{2/3})$& $\Omega(\sqrt{T})$ \\
FSW18& parametric& $\tilde O(\sd^2\sqrt{KT})$ \textsuperscript{*}& N/A\\
CJD19& non-parametric, $f\in\Sigma_{\sd}(\infty,M)\cap\Gamma_{\sd}(\sigma)$&$O(T^{1/2+\epsilon})$, $\forall\epsilon>0$ & N/A \\
MWZ21& parametric& $\tilde O(\sd^{3.5}\sqrt{T})$& N/A\\
MW21& non-parametric, $f\in\Sigma_{\sd}(2,M)\cap\Gamma_{\sd}(\sigma)$& $\tilde O(\sd^{3.5}\sqrt{T})$ \textsuperscript{$\dagger$}& N/A \\ 
\hline
\textbf{This paper}& non-parametric,$f\in\Sigma_{\sd}(2,M)\cap\Gamma_{\sd}(\sigma)$& $\tilde O(\sd^{2.25}\sqrt{T})$  \textsuperscript{$\dagger$}& $\Omega(\sqrt{T})$\textsuperscript{$\sharp$}  \\
\hline
\end{tabular}
}
\vskip 0.2in
\begin{flushleft}
{\footnotesize \textsuperscript{$*$}Bayesian regret and finite numbe ($K$) of price candidates.

\textsuperscript{$\dagger$} Hiding additional terms depending polynomially on $\sd$ and growing slower than $\tilde O(\sqrt{T})$.

\textsuperscript{$\sharp$} Not proved in this paper but obtained by invoking lower bound results in existing literature \citep{agarwal2010optimal}.}
\end{flushleft}
\end{table}

In addition to works summarized in Table \ref{tab:related-works-nrm}, other relevant works include \citep{WDY2014,BZ2009,wang2019multi,besbes2015surprising,cheung2017dynamic,nambiar2019dynamic,bu2022online}
studying simplified or variants of the NRM model, such as multi-modal demand functions, power of linear models, model mis-specification
and incorporation of offline data.
More broadly, the problem originates from the seminal works of \cite{gallego1994optimal} analyzing
optimal dynamic pricing strategies and fixed-price heuristics with stochastic, multi-period demands.

\end{itemize}

\section{Problem formulation and assumptions}\label{sec:formulation}

\subsection{The pairwise comparison oracle}

We first give a formal definition of the pairwise comparison oracle that will be studied in this paper.
To ensure smooth applications to important OM and business questions, we will state the comparison oracle in full technical generality
so that the results derived in this paper are as general as possible.
\begin{definition}
A pairwise comparison oracle $\mathcal O$ associated with an underlying function $f:\mX\to\mathbb R$ is \emph{$\gamma_1,\gamma_2$-consistent}
if the following holds: for any $\delta\in(0,1)$, $n\geq 1$, $x,x'\in\mX$, the oracle $\mathcal O$ consumes $n$ time periods and returns
an estimate $y\in\mathbb R$ that satisfies
$$
\Pr\left[\big|y-(f(x')-f(x))\big| > \sqrt{\frac{\gamma_1+\gamma_2\ln(1/\delta)}{n}}\right]\leq \delta.
$$
\label{defn:oracle}
\end{definition}

In Definition \ref{defn:oracle}, we define a pairwise comparison oracle $\mathcal O$ that can be used to indirectly obtain information
about an unknown objective functrion $f$ that is to be optimized.
This oracle is in general \emph{biased}, only conferring meaning information of $f$ when invoked with a sufficiently large number of time periods $n$,
making it significantly different from most stochastic/noisy feedback or dueling bandit oracles that contains unbiased information of objective functions
with only a single query/arm pull.
Throughout the rest of this paper, we use the notation $\mathcal O(n,x,x')$ to denote the output of the pairwise comparison oracle
when it is invoked to compare $x$ and $x'$ with $n$ time periods.

The pairwise comparison oracle defined in Definition \ref{defn:oracle} encompasses several important feedback structures as special cases,
as we enumerate below:
\begin{example}[The standard bandit feedback]
In the standard bandit feedback setting the algorithm supplies $x_t\in\mX$ and observes $z_t=f(x_t)+\xi_t$, with $\mathbb E[\xi_t|x_t]=0$ and $|\xi_t|\leq B_\xi$ almost surely. The oracle $\mathcal O$ can be easily constructed using the difference of the sample averages at $x_t$ and $x_t'$.
The standard Hoeffding's inequality \citep{hoeffding1963probability} implies that the constructed oracle $\mathcal O$ is $\gamma_1,\gamma_2$-consistent
with $\gamma_1=2B_\xi^2\ln2$ and $\gamma_2=2B_\xi^2$. 
\end{example}

\begin{example}[The continuous dueling bandit]
In the continuous dueling bandit setting \citep{kumagai2017regret},
at time $t$ the algorithm supplies $x_t,x_t'\in\mX$ and observes a binary comparison feedback $z_t\in\{0,1\}$ such that $\Pr[z_t=1|x_t,x_t'] = 1/2 + \eta(f(x_t')-f(x_t))$ for some known link function $\eta$.
Using the Hoeffding's inequality this implies that, with $n$ samples, the sample average $\bar z=\frac{1}{n}\sum_{t=1}^n z_t$ satisfies $\Pr[|\bar z-1/2-\eta(f(x_t')-f(x_t))|\leq \sqrt{\ln(2/\delta)/2n}$ with probability $1-\delta$.
Subsequently, if $\eta$ is invertible and $\eta^{-1}$ is $L_\eta$-Lipschitz continuous, then $\mathcal O(n,x,x')=\eta^{-1}(\bar z-1/2)$ satisfies Definition \ref{defn:oracle}
with $\gamma_1=L_\eta\sqrt{\ln2 /2}$ and $\gamma_2 = L_\eta/\sqrt{2}$.
\end{example}

In addition to the above examples, it is also common that (noisy) pairwise comparison oracles could be constructed for more complex operations management problems.
In later Sec.~\ref{sec:om} of this paper we mention two important operations management problems for which pairwise comparison oracles could be rigorously constructed,
albeit through quite complex algorithms and procedures.

\subsection{Function classes and assumptions}\label{subsec:function-classes}

It is clear that, if the underlying, unknown function $f:\mX\to\mathbb R$ does not satisfy any regularity properties, it is mathematically impossible
to design any non-trivial optimization algorithms. In this section, we introduce two standard function classes that impose certain smoothness and shape constraints
on the underlying objective function $f$ so that the optimization problem is tractable.

\begin{definition}[The H\"{o}lder class]
For $k\in\mathbb N$ and $M<\infty$, the H\"{o}lder function class $\Sigma_d(k,M)$ over $\mX$ is defined as
$$
\Sigma_d(k,M) := \left\{ f\in \mathcal C^{k}(\mX): \sup_{x\in\mX}\left|\frac{\partial^{k_1+\cdots+k_d}f(x)}{\partial x_1^{k_1}\cdots\partial x_d^{k_d}}\right|\leq M, \;\;\forall k_1+\cdots+k_d\leq k\right\},
$$
where $\mathcal C^k(\mX)$ is the set of all $k$-times continuously differentiable functions on $\mX$.
\label{defn:holder}
\end{definition}

The H\"{o}lder function class is a popular function class that measures the \emph{smoothness} of a function with higher orders $k$ and/or smaller constants $M$ indicating
smoother functions. 
In the most general definition of the H\"{o}lder class the order $k$ could be further extended to any strictly positive real numbers.
We shall however restrict ourselves to only integer-valued orders to make our analysis simpler.

We also consider functions that are strongly concave on $\mX$, as specified in the following definition:
\begin{definition}[Strongly concave functions]
For $\sigma>0$, the strongly concave function class $\Gamma_d(\sigma)$ over $\mX$ is defined as
$$
\Gamma_d(\sigma) := \left\{f\in\mathcal C^2(\mX): -\nabla^2 f(x)\succeq \sigma I_{d\times d}, \;\;\forall x\in\mX\right\}.
$$
\label{defn:sc}
\end{definition}

Definition \ref{defn:sc} above captures strongly concave functions that are at least twice continuously differentiable, by constraining the Hessian matrices of $f$ to be negative definite
with all eigenvalues uniformly bounded away from zero on $\mX$.
Strongly concave functions naturally arise in many application problems and we show in this paper that for strongly concave and smooth functions 
there is an improved algorithm with improved performance guarantees as well.


\subsection{Domain and knapsack constraints}

We make the following assumption on the domain $\mX$ of the objective function $f$:
\begin{assumption}[Interior maximizer]
There exists a constant $\delta_0>0$ such that $\{x: \|x-x^*\|_{\infty}\leq\delta_0\}\subseteq\mX$,
where $x^*=\min_{x\in\mX}f(x)$.
\label{asmp:interior}
\end{assumption}
Such interior properties of the maximizer $x^*$ are common assumptions imposed for stochastic zeroth-order
or bandit optimization problems \citep{besbes2015non}.

In addition to optimizing an unknown objective function stipulated in Definition \ref{defn:oracle}
and hard domain constraints $x\in\mX$, in many applications such as network revenue management,
an ``averaging'' knapsack constraint is also important, preventing the algorithm from making too many decisions that are far away from a
certain feasibility region dictated by initial inventory constraints.
To this end, we introduce an averaging ``penalty'' function $\psi:\mX\to[0,\infty)$ satisfying the
following condition:
\begin{assumption}[Knapsack penalty function]
The penalty function $\psi:\mX\to[0,\infty)$ is convex and satisfies $\psi(x)=0,\forall x\in\mZ$
for some convex compact feasible region $\mZ\subseteq\mX$ containing $x^*$.
Furthermore, there exists a constant $L_\psi<\infty$ such that $|\psi(x)-\psi(x')|\leq L_\psi\|x-x'\|_{\infty}$ for all $x,x'\in\mX$.
\label{asmp:knapsack}
\end{assumption}

{
\begin{example}[Network revenue management]
In NRM, the opitmization variable is the expected demand rates for $\sd$ product types, $x\in\mathbb R_{++}^\sd$,
and the feasible region is $Ax\leq\gamma$ where $A\in\mathbb R^{\sd'\times \sd}$ is the resource consumption matrix and $\gamma\in\mathbb R_+^{\sd'}$
is the normalized initial inventory levels.
For this particular example, the penalty function could be defined as $\psi(\bar x)=\bar p \sd A_{\min}^{-1}\max_{1\leq j\leq \sd'}\{0, (A\bar x)_j-\gamma_j\}$
which penalizes over-selling of products with depleted resource types,
where $\bar p$ is the maximum price that could be offered and $A_{\min}$ is the smallest non-zero entry in the resource consumption matrix $A$.
This penalty function satisfies Assumption \ref{asmp:knapsack} with respect to the feasible region $\{x\in\mathbb R_{++}^{\sd'}: Ax\leq\gamma\}$ with 
parameter $L_\psi = \bar p\sd A_{\min}^{-1}\times\max_{1\leq j\leq \sd'}\|A_{j\cdot}\|_1$.
\label{exmp:nrm}
\end{example}
}

\subsection{Admissible policies and regret}

In this section we give a rigorous definition of an admissible policy (that is, policies that are non-anticipating and only use past observations
to guide future actions). An admissible policy $\pi$ consists of a variable sequence of conditional distributions $\pi=(\pi_1,\pi_2,\cdots)$ 
such that for $\tau\geq 1$, 
$$
n_\tau, x_\tau,x_\tau' \;\;\sim\;\; \pi_\tau(\cdot|y_1,n_1,x_1,x_1',\cdots,y_{\tau-1},n_{\tau-1},x_{\tau-1},x_{\tau-1}'),\;\;\;\;\;\text{and}\;\;
y_\tau \sim \mathcal O(n_\tau, f(x_\tau),f(x_\tau')).
$$
Furthermore, with probability 1 there exists $\tau_0$ such that $n_1+\cdots+n_{\tau_0}=T$.

The cumulative regret of an admissible policy $\pi$ is defined as the expectation of the differences between the optimal (maximum) objective value
and the actions the policy $\pi$ takes:
\begin{equation}
\mathbb E^\pi\left[\left(\sum_{\tau=1}^{\tau_0} n_\tau (2f^*-f(x_\tau)-f(x_\tau'))\right) + T\psi\left(\frac{1}{2T}\sum_{\tau=1}^{\tau_0}n_\tau(x_\tau+x_\tau')\right)\right],
\label{eq:defn-regret}
\end{equation}
where $\tau_0$ is the unique integer such that $n_1+\cdots+n_{\tau_0}=T$ which exists almost surely, and $f^*=\max_{x\in\mX}f(x)$.  { Eq. \eqref{eq:defn-regret} is a natural definition of the regret. To see this, each query of the pair $(x_\tau, x_\tau')$ generates a reward of $f(x_\tau)+f(x_\tau')$. As the maximum possible reward for querying a pair is $2f^*$, the regret for this query is $2f^*-f(x_\tau)-f(x_\tau')$. Recall that the pair $(x_\tau, x_\tau')$ will be queried for $n_\tau$ periods, which gives a natural regret definition in \eqref{eq:defn-regret}.

}

\section{Algorithm for smooth functions}\label{sec:alg-smooth}

We now give an overview of our algorithm for smooth functions, that is, those functions that belong to H\"{o}lder smoothness classes $\Sigma_d(k,M)$ but may not be strongly concave
or have other particular structures/shapes.
The main algorithm, pseudocode displayed in Algorithm \ref{alg:main-tournament}, is based on the following two main ideas:
\begin{enumerate}
\item The entire solution space $\mX$ is partitioned into multiple smaller cubes; in each cube, a iterative LinUCB method is applied to figure out the approximately best solution \emph{within that cube};
\item For competitions across tubes, a successive tournament elimination idea is applied, eliminating nearly half the small cubes every iteration via the noisy pairwise comparison oracle, until a unique winner is selected.
\end{enumerate}
The first idea (iterative UCB method to find an approximate best solution within a small cube) is further decomposed into two sub-routines, elaborated in Secs.~\ref{subsec:lpa} and \ref{subsec:iterative-linucb}.
The procedure relies on local polynomial regression to exploit the smoothness of the objective function, and an iterative application of batched LinUCB algorithm
so that the standard LinUCB algorithm could be adapted to our setting where only a noisy pairwise comparison oracle is available, instead of noisy but unbiased function value observations.

\subsection{Local polynomial approximation}\label{subsec:lpa}

To exploit the smoothness of the objective function $f$, it is conventional to use local polynomial approximation schemes
to approximate $f$ across multiple smaller cubes \citep{wang2019optimization,fan2018local,wang2019multi,denboer2024pricing}.
More specifically, let $J\in\mathbb N$ be an algorithm parameter and divide $[0,1]^d$ into $J^d$ cubes, each of size $J^{-1}\times\cdots\times J^{-1}$.
For any vector $\vct j\in[J]^d$, let 
$$
\mathcal X_{\vct j} := \left(\left[\frac{j_1-1}{J}, \frac{j_1}{J}\right]\times\cdots\times\left[\frac{j_d-1}{J},\frac{j_d}{J}\right]\right) \cap \mX\cap\mZ
$$
be a cube corresponding to the indicator vector $\vct j$.
Let $x_{\vct j}$ be an arbitrary point in $\mathcal X_{\vct j}$. 
For any $x\in\mathcal X_{\vct j}$ and the smoothness level $k\in\mathbb N$, define the local polynomial map $\phi_{\vct j}:\mathcal X_{\vct j}\to\mathbb R^{\nu}$,
$\nu=\binom{k+d-1}{d}$, as
$$
\phi_{\vct j}(x) = \left[\prod_{i=1}^d (x_i-x_{\vct ji})^{k_i}\right]_{k_1+\cdots+k_d< k}.
$$
For example, with $d=2$ and $k=3$, the feature map $\phi_{\vct j}(\cdot)$ can be explicitly written as $\phi_{\vct j}(x)=[1, x_1-x_{\vct d1}, x_2-x_{\vct d2}, x_3-x_{\vct d3}, (x_1-x_{\vct d1})^2, (x_1-x_{\vct d1})(x_2-x_{\vct d2}), (x_1-x_{\vct d1})(x_1-x_{\vct d3}), (x_2-x_{\vct d2})^2, (x_2-x_{\vct d2})(x_3-x_{\vct d3}), (x_3-x_{\vct d3})^2]$.

The following lemma upper bounds the error of using local polynomial mapping $\phi_{\vct j}$ as an approximation of $f$.
\begin{lemma}
Let $f\in\Sigma_d(k,M)$. Then for any $\vct j\in[J]^d$, there exists $\theta_{\vct j}\in\mathbb R^{\nu}$, $\|\theta_{\vct j}\|_2\leq M\sqrt{\nu}$,
such that
$$
\max_{x\in\mathcal X_{\vct j}}\big|f(x)-\langle\theta_{\vct j},\phi_{\vct j}(x)\rangle\big|\leq (d+k)^k M\times J^{-k}.
$$
\label{lem:local-poly-approx}
\end{lemma}
Lemma \ref{lem:local-poly-approx} shows that the approximation error of $f$ using local polynomial regression decreases as the number of small partitioning cubes $J$ increases,
leading to more local approximations. Furthermore, the order of smoothness $k$ also plays an important role in the upper bound of approximation errors,
as smoother functions admit smaller approximation errors. Lemma \ref{lem:local-poly-approx} can be proved by multi-variate Taylor expansion with Lagrangian remainders,
which we place in the supplementary material.

\subsection{Linear bandit with batched updates and biased rewards}\label{subsec:iterative-linucb}


Algorithm \ref{alg:batch-linucb-comparison} gives a pseudo-code description of the LinUCB algorithm that will output approximate best solutions in a small cube.
From Lemma \ref{lem:local-poly-approx}, we know that locally $f$ could be well approximated by a \emph{linear function} with respect to the polynomial map of centered solution vectors:
this gives the foundation of using LinUCB and linear bandit methods.

\begin{algorithm}[t]
\caption{The batch Lin-UCB algorithm with pairwise comparison oracles}
\label{alg:batch-linucb-comparison}
\begin{algorithmic}[1]
\Function{BatchLinUCB}{$\vct j,\xs,N,C_1,C_2$}
	\State $\Lambda\gets I_{\nu\times\nu}$, $\hat\theta\gets (0,\cdots,0)\in\mathbb R^\nu$, $C_1'=2M\sqrt{\nu}+2C_2\sqrt{\tau_{\infty}N} +2\sqrt{C_1\tau_{\infty}}$, where $\tau_{\infty}=\nu\log_2(2N\nu)$;
	\For{$\tau=1,2,\cdots$ until $N=0$}
		\State $x_\tau \gets \arg\max_{x\in\mathcal X_{\vct j}}\min\{M, \langle\hat\theta, \phi_{\vct j}(x)-\phi_{\vct j}(\xs)\rangle + C_1'\sqrt{(\phi_{\vct j}(x)-\phi_{\vct j}(\xs))^\top\Lambda^{-1}(\phi_{\vct j}(x)-\phi_{\vct j}(\xs))}+C_2\}$; \label{line:batch-ucb}
		\State $\phi_\tau\gets\phi_{\vct j}(x_\tau)-\phi_{\vct j}(\xs)$;
		\State Let $n_\tau\leq N$ be the smallest integer such that $\det(\Lambda+n_\tau\phi_\tau\phi_\tau^\top)> 2\det(\Lambda)$;
		\State Let $y_\tau$ be the output of the comparison oracle $\mathcal O$ invoked with $n_\tau,x_\tau,\xs$;
		\State $N\gets N-n_\tau$, $\Lambda\gets\Lambda+n_\tau\phi_\tau\phi_\tau^\top$, $\hat\theta\gets\arg\min_{\|\theta\|_2\leq M\sqrt{\nu}}\sum_{j\leq\tau}n_j(y_j-\langle\theta,\phi_j\rangle)^2+\|\theta\|_2^2$;
	\EndFor
	\State \textbf{return} $x_{\tau^*}$ where $\tau^*=\arg\max_{\tau}n_\tau$;
\EndFunction
\end{algorithmic}
\end{algorithm}

At a higher level, Algorithm \ref{alg:batch-linucb-comparison} resembles the LinUCB algorithm with infrequent policy updates, analyzed in \citep{abbasi2011improved}.
However, given our problem structure, there are several important differences:
\begin{enumerate}
\item In our problem we only have access to a (biased) pairwise comparison oracle, instead of noisy unbiased observation of rewards for a given action.
In Algorithm \ref{alg:batch-linucb-comparison}, we use a ``baseline solution'' $\xs$ to support pairwise comparisons, motivated by the observation that $f(x)-f(\xs) \approx \langle \phi(x)-\phi(\xs),\theta\rangle$
remains a linear model for $x\approx \xs$ locally thanks to Lemma \ref{lem:local-poly-approx}.

The use of such a baseline $\xs$ also requires that $f(\xs)$ is not too sub-optimal, because otherwise invocations of $\mathcal O(n,x,\xs)$ would incur large regret on the $\xs$ part.
The near-optimality of $\xs$ is ensured in Algorithm \ref{alg:iterative-batch-lin-ucb} via an iterative procedure with double epochs, which we discuss in details in the next section.

\item In our problem, because of the intrinsic bias in the pairwise comparison oracle, the observed feedback $\{y_\tau\}_\tau$ is \emph{not} unbiased. This requires us to take extra care in the setting of the $C_1'$ parameter in confidence terms to make sure the optimal solution is not overlooked.

\item The output of Algorithm \ref{alg:batch-linucb-comparison} is a single solution in $\mX$ so that it can be compared against (near-optimal) solutions from other cubes in a successive tournament elimination procedure.
This is in contrast to standard linear bandit which only needs to ensure low cumulative regret from a sequence of solutions/actions.
In our algorithm, we use the solution that lasts for the longest amount of time as the final output, which serves as an approximate good solution thanks to the pigeon hole principle.
\end{enumerate}


Our next lemma analyzes Algorithm \ref{alg:batch-linucb-comparison}.

\begin{lemma}
Let $x_{\vct j}^* = \arg\max_{x\in\mathcal X_{\vct j}}f(x)$ and $f_{\vct j}^*=f(x_{\vct j}^*)$.
Suppose Algorithm \ref{alg:batch-linucb-comparison} is run with $C_1=\gamma_1+2\gamma_2\ln T$ and $C_2=(d+k)^k M J^{-k}$.
Then with probability $1-\widetilde O(T^{-2})$ the following properties hold:
\begin{enumerate}
\item $f_{\vct j}^* - f(x_{\tau^*}) \leq 3C_2 \nu \ln(2\nu N) + 12C_1'M\sqrt{\nu^3\ln^2(5\nu N)}/\sqrt{N}$;
\item $\tau_0\leq \tau_{\infty}=\nu\log_2(2N\nu)$;
\item $\sum_{\tau=1}^{\tau_0} n_\tau(f_{\vct j}^*-f(x_{\tau}))\leq 2C_2N + 8C_1'M\sqrt{\nu N\ln(5\nu N)}$.
\end{enumerate}
\label{lem:batch-linucb-regret}
\end{lemma}

Lemma \ref{lem:batch-linucb-regret} contains three results, each serving an important purpose.
The first result upper bounds the sub-optimality gap between the optimal objective value $f$ in cube $j$ and the solution $x_\tau^*$, the final output of Algorithm \ref{alg:batch-linucb-comparison},
showing that $x_\tau^*$ is an approximate optimizer of $f$ in cube $j$. 
The second result upper bounds the total number of iterations $\tau$ in Algorithm \ref{alg:batch-linucb-comparison}, which is important in establishing that the algorithm makes infrequence updates
to solutions.
The third result upper bounds the cumulative regret incurred throughout the entire algorithm.
Notably, it only upper bounds \emph{half} of the regret incurred by the pairwise comparison oracle on the parts of $x_\tau$:
the other half of the regret will be upper bounded in the next section after we introduce an iterative invocation procedure to ensure the near-optimality of $\xs$ solutions.

Due to space constraints, the complete proof of Lemma \ref{lem:batch-linucb-regret} is placed in the supplementary material.

\begin{remark}
For theoretical analysis, the values of several algorithm parameters $C_1,C_2,C_1'$ require knowledge of the smoothness parameters $k$ and $M$.
Unfortunately, such prior knowledge is mathematically {necessary} in order to design and analyze effective algorithms,
as demonstrated by the negative result of \cite{locatelli2018adaptivity}.
In practice, we recommend the use of $k=3$ for reasonably smooth functions, $k=4$ for very smooth functions,
and the value of $M=\ln T$ for the derivative upper bounds.
\end{remark}

\begin{remark}
The constants in the definition of $C_1'$ in Algorithm \ref{alg:batch-linucb-comparison} are for theoretical analytical purposes
and are therefore conservative for practical use.
When implementing the algorithm, together with the recommendation made in the previous remark,
we recommend the choice of $C_1'=\sqrt{\nu}\ln T+C_2\sqrt{\tau_{\infty}N} + \sqrt{C_1\tau_{\infty}}$,
with $C_1=\gamma_1+\gamma_2\ln T$ and $C_2=(d+k)^k J^{-k}\ln T$.
\end{remark}

\subsection{Iterative applications of batched linear bandit}

The \textsc{BatchLinUCB} routine introduced and analyzed in the previous section requires a fixed anchoring point $\xs$
to facilitate pairwise comparisons. Clearly, if $f(\xs)$ is small the cumulative regret incurred by \textsc{BatchLinUCB} would be large
regardless of how well it learns the maximum of $f$ in $\mathcal X_{\vct j}$.
This is evident from the last property of Lemma \ref{lem:batch-linucb-regret}, where only half the regret incurred in \textsc{BatchLinUCB}
is upper bounded.
In this section we introduce an iterative application of the \textsc{BatchLinUCB} routine (pseudocode description in Algorithm \ref{alg:iterative-batch-lin-ucb}),
which controls the other half of the total regret incurred in the LinUCB procedure from $\xs$.

The basic idea behind Algorithm \ref{alg:iterative-batch-lin-ucb} is simple: we use a doubling trick in the outer iteration to make sure that the $\xs$ solution
provided is directly from the optimal solution output by Algorithm \ref{alg:batch-linucb-comparison} in the previous iteration.
This ensures that for later iterations that are longer, the quality of $\xs$ is also higher because it is obtained as an approximate optimizer in the previous iteration.

\begin{algorithm}[t]
\caption{Iterative application of the \textsc{BatchLinUCB} sub-routine.}
\label{alg:iterative-batch-lin-ucb}
\begin{algorithmic}[1]
\Function{IterativeBatchLinUCB}{$\vct j,N,C_1,C_2$}
	\State Let $x_0$ be an arbitrary point in $\mathcal X_{\vct j}$ and $\beta_0=\lfloor\log_2(N+1)\rfloor - 1$;
	\For{$\beta=1,2,\cdots,\beta_0$}
		\State $x_{\beta} \gets \textsc{BatchLinUCB}(\vct j,x_{\beta-1},2^\beta,C_1,C_2)$;
	\EndFor
	\State For the remaining of $N$ periods commit to action $x_{\beta_0},x_{\beta_0}$;
	\State \textbf{return} $x_{\beta_0}$;
\EndFunction
\end{algorithmic}
\end{algorithm}

\begin{lemma}
Let $x_{\vct j}^* = \arg\max_{x\in\mathcal X_{\vct j}}f(x)$ and $f_{\vct j}^*=f(x_{\vct j}^*)$.
Suppose Algorithm \ref{alg:iterative-batch-lin-ucb} is run with $C_1=\gamma_1+2\gamma_2\ln T$ and $C_2=(d+k)^k M J^{-k}$.
Then with probability $1-\widetilde O(T^{-2})$ the following properties hold:
\begin{enumerate}
\item $f_{\vct j}^*-f(x_{\beta_0})\leq (34M+42\sqrt{C_1})M\nu^2\sqrt{\ln^3(5\nu N)}/\sqrt{N} + (6\nu+58M)C_2\nu^2 \ln^{1.5}(5\nu N)$;
\item $\sum_{t=1}^N 2f_{\vct j}^*-f(x_t)-f(x_t')\leq 56C_2 M\nu^2 N\ln^2(5\nu N) + 42M\nu^2\sqrt{C_1 N}\ln^2(5\nu N)$, where $x_t,x_t'\in\mathcal X_{\vct j}$ are pairs of actions taken at time $t$ over the $N$ time periods
involved in Algorithm \ref{alg:iterative-batch-lin-ucb}.
\end{enumerate}
\label{lem:iterative-batch-linucb-regret}
\end{lemma}

Lemma \ref{lem:iterative-batch-linucb-regret} is proved in the supplementary material. It shows that the total regret of Algorithm \ref{alg:batch-linucb-comparison},
when invoked in an iterative manner with geometrically increasing number of periods, admits cumulative regret that is properly upper bounded.


\subsection{Main algorithm: tournament successive eliminations}

We are now ready to present our main algorithm, with pseudo-code in Algorithm \ref{alg:main-tournament}.

\begin{algorithm}[t]
\caption{Main algorithm: tournament successive eliminations}
\label{alg:main-tournament}
\begin{algorithmic}[1]
\State \textbf{Input}: time horizon $T$, pairwise comparison oracle $\mathcal O$, parameters $k,J,M,C_1,C_2$;
\State $\mathcal A_1\gets [J]^d$, $C_2'\gets (12\nu+116M)C_2\nu^2\ln^{1.5}(5\nu T)$,  $C_3\gets(68M+85\sqrt{C_1})^2M^2\nu^4\ln^3(5\nu T)$;
\For{$\zeta=1,2,3,\cdots$ or until $T$ time periods have been reached}
	\State $\varepsilon_\zeta\gets 2^{-\zeta}$, $N_\zeta\gets \lceil C_3/\varepsilon_\zeta^2\rceil$, $N_\zeta^\tot\gets 3(|\mathcal A_\zeta|-1)N_\zeta$;
	\If{fewer than $N_\zeta^\tot$ time periods remain}
		\State Pick arbitrary $x,x'\in \mathcal X_{\vct j}$, $\vct j\in\mathcal A_\zeta$ and commit to them for the remaining time periods;
	\Else
		\State For each $\vct j\in\mathcal A_\zeta$, let $x_{\vct j,\zeta} \gets \textsc{IterativeBatchLinUCB}(\vct j,N_\zeta,C_1,C_2)$;
		$\mathcal A_\zeta^{(1)}\gets\mathcal A$;\label{line:step-preprocess}
		\For{$\omega=1,2,\cdots$ until $|\mathcal A_\zeta^{(\omega)}|>1$}
			\For{\textsuperscript{*}each pair $\vct j,\vct j'\in\mathcal A_\zeta^{(\omega)}$}
				\State$y\gets\mathcal O(N_\zeta,x_{\vct j,\zeta},x_{\vct j',\zeta})$;\label{line:step-tournament}
				\State $\mathcal A_\zeta^{(\omega+1)}\gets\mathcal A_\zeta^{(\omega+1)}\cup\{\vct j'\}$ if $y\geq 0$ and
			$\mathcal A_\zeta^{(\omega+1)}\gets\mathcal A_\zeta^{(\omega+1)}\cup\{\vct j\}$ if $y<0$;\label{line:step-Azeta}
			\EndFor
		\EndFor
		\State Let $\hat{\vct j}(\zeta)$ be the only element in $\mathcal A_\zeta^{(\omega)}$; $\mathcal A_{\zeta+1}\gets\{\hat{\vct j}(\zeta)\}$;
		\For{each $\vct j\in\mathcal A_\zeta\backslash\{\hat{\vct j}(\zeta)\}$}
			\State $y\gets \mathcal O(N_\zeta, x_{\vct j,\zeta},x_{\hat{\vct j}(\zeta),\zeta})$;
			\State If $y\leq \varepsilon_\zeta+C_2'$ then let $\mathcal A_{\zeta+1}\gets\mathcal A_{\zeta+1}\cup\{\vct j\}$;\label{line:step-elimination}
		\EndFor
	\EndIf
\EndFor
\State If there are time periods remaining, commit to $x_{\vct j,\zeta},x_{\vct j,\zeta}$ for an arbitrary $\vct j\in\mathcal A_{\zeta}$;
\end{algorithmic}
{\footnotesize\textsuperscript{*} If $|\mathcal A_\zeta^{(\omega)}|$ is odd then transfer one arbitrary $\vct j\in\mathcal A^{(\omega)}$ directly to $\mathcal A^{(\omega+1)}$;}
\end{algorithm}

Algorithm \ref{alg:main-tournament} is built upon the idea of successive tournament eliminations.
More specifically, each cube $\vct j\in[J]^d$ is regarded as a ``competitor" and the competitors are paired with each other to play a hypothetical tournament;
those cubes who ``won'' pairwise competitions are advanced to the next round, where each round is indicated by a variable $\omega$ and the ``active" cubes/competitors
at the beginning of round $\omega$ is denoted as a set $\mathcal A^{(\omega)}$ (Lines \ref{line:step-tournament} and \ref{line:step-Azeta}). Finally, the winner of this tournament is used to 
determine whether other competitors should be kept for the next outer iteration, as shown on Line \ref{line:step-elimination}.

The following lemma is a structural lemma that plays a crucial in our understand and analysis of Algorithm \ref{alg:main-tournament}.
\begin{lemma}
Suppose Algorithm \ref{alg:main-tournament} is run with $C_1=\gamma_1+2\gamma_2\ln T$ and $C_2=(d+k)^kMJ^{-k}$.
Let also $\vct j^*=\arg\max_{\vct j\in [J]}f_{\vct j}^*=\arg\max_{\vct j\in[J]}\{\max_{x\in\mathcal X_{\vct j}}f(x)\}$.
Define also $\bar\omega := 1.45\ln(2T)$.
Then with probability $1-\widetilde O(T^{-2})$ the following hold for all $\zeta$:
\begin{enumerate}
\item $\vct j^*\in\mathcal A_\zeta$;
\item For every $\vct j\in\mathcal A_\zeta$, $f_{\vct j^*}^*-f_{\vct j}^*\leq 3\bar\omega(\varepsilon_\zeta+C_2')$.
\end{enumerate}
\label{lem:tournament-keys}
\end{lemma}

Lemma \ref{lem:tournament-keys} is proved in the supplementary material. It outlines the most important structures and properties of Algorithm \ref{alg:main-tournament}.
The first property, $\vct j^*\in\mathcal A_\zeta$, indicates that (with high probability) the algorithm will never eliminate the cube with the optimal solution
during any rounds of the tournament elimination phase.
The second property states that, with high probability, all competitors who are still active at outer iteration $\zeta$ after the tournament and elimination steps
are nearly optimal, especially for later outer iterations that last longer.

We are now ready to present the main theoretical results upper bounding the cumulative regret of Algorithm \ref{alg:main-tournament}:
\begin{theorem}
Suppose $f\in\Sigma_d(k,M)$ and Algorithm \ref{alg:main-tournament} is run with parameters $C_1=\gamma_1+2\gamma_2\ln T$, 
$C_2=(d+k)^kMJ^{-k}$ and $J=\lceil T^{1/(2k+d)}\rceil$. Then there exists a constant $\overline C_1<\infty$ depending polynomially on $k$ and $d$ 
such that with probability $1-\tilde O(T^{-2})$, the cumulative regret of Algorithm \ref{alg:main-tournament} is upper bounded by 
$$
\overline C_1\times (M^2\nu^2+\nu^3+M\nu^2\sqrt{\gamma_1+\gamma_2\ln T})\times T^{\frac{k+d}{2k+d}}\ln^4 T.
$$
\label{thm:regret-smooth-functions}
\end{theorem}

Theorem \ref{thm:regret-smooth-functions} shows that for objective functions of dimension $d$ that belong to the H\"{o}lder smoothness class of order $k$,
the asymptotic regret scaling is $\tilde O(T^{\frac{k+d}{2k+d}})$ as time horizon $T$ increases.
The regret upper bound is minimax optimal in time horizon $T$ up to poly-logarithmic terms,
because it matches existing lower bound on continuum-armed bandit in which unbiased function value observations are available \citep{wang2019optimization},
a special case of the setting studied in this paper.

\begin{remark}
The large constants in the $C_2',C_3$ parameters in Algorithm \ref{alg:main-tournament} is for theoretical analysis only and in practical implementations they should be 
replaced with constants of one.
\end{remark}

\section{Algorithm for strongly concave functions}

In this section, we design algorithms for a special case of the general bandit question with noisy pairwise comparison oracles,
in which the objective function to be maximized is \emph{strongly concave}.
The (strong) concavity of objective functions is a common condition in operations management practice, such as ``regularity'' of demand/revenue functions in dynamic pricing \citep{chen2019nonparametric},
and the strong concavity of inventory holding/lost-sales costs when demand distributions are non-degenerate \citep{huh2009nonparametric}.

Mathematically, throughout this section we assume that 
\begin{equation}
f\in \Sigma_d(2,M)\cap \Gamma_d(\sigma)
\label{eq:asmp-sc}
\end{equation}
for some constants $M<\infty$ and $\sigma>0$, where the function classes $\Sigma_d(k,M)$ and $\Gamma_d(\sigma)$ are defined in Sec.~\ref{subsec:function-classes}.

\subsection{Proximal gradient descent with inexact gradients}


\begin{algorithm}[t]
\caption{A biased proximal gradient descent algorithm}
\label{alg:acc-pgd}
\begin{algorithmic}[1]
\State \textbf{Input}: time horizon $T$, pairwise comparison oracle $\mO$, parameters $\eta,\alpha,\sigma$;
\State Initialization: an arbitrary $x_0\in\mX^o\cap\mZ$ where $\mX^o=\{x: x'\in\mX,\forall \|x'-x\|_{\infty}\leq\delta_0\}$; 
\For{$\tau=0,1,2,\cdots$}
	\State $\beta_\tau \gets \lceil (1+\eta)^{\tau}\rceil$, $h_\tau \gets \sqrt[4]{(\gamma_1+2\gamma\ln T)/(\beta_\tau d)}$;
	\State If there are fewer than $4\beta_\tau$ time periods left, commit to $x_\tau$ for the rest of the time periods;
	\For{$j=1,2,\cdots,d$}
		\State Let $ x_\tau'\gets x_\tau + h_\tau e_j$ and invoke the oracle with $y_\tau(j) \gets \mO(\beta_\tau,  x_\tau,  x_\tau')$;\label{line:xtaup}
		\State Let $ x_{\tau}''\gets x_\tau-h_\tau e_j$ and invoke the oracle with $y_\tau(j)'\gets \mO(\beta_\tau, x_\tau, x_\tau'')$;\label{line:xtaupp}
		\State $\hat g_\tau(j) \gets (y_\tau(j)-y_\tau(j)') / 2h_\tau+\sigma x_{\tau,i}$;
	\EndFor
	\State $\hat{\vct g}_\tau \gets (\hat g_\tau(1),\cdots,\hat g_\tau(d))$;
	\State $x_{\tau+1} \gets \arg\min_{u\in\mX^o\cap \mZ}\{-\langle\hat{\vct g}_\tau, u\rangle + \sigma\|u\|_2^2/2 + \alpha^{-1}\|u-x_\tau\|_2^2/2\}$;\label{line:iter-pgd}
\EndFor
\end{algorithmic}
\end{algorithm}

The additional shape constraint $f\in\Gamma_d(\sigma)$ calls for new algorithmic procedures that could exploit the strong concavity of $f$.
To this end, we design a proximal gradient descent algorithm with batch inexact gradient evaluations, as described in Algorithm \ref{alg:acc-pgd}.

In the outer iterations of Algorithm \ref{alg:acc-pgd}, we implement a projected (inexact) gradient ascent approach. Fixed step sizes are used to exploit strong concavity
of the objective function, and the the iterations have geometrically increasing number of time periods dedicated for inner iterations in order to properly maintain gradient estimation error
in the finite-difference approach (Lines \ref{line:xtaup} and \ref{line:xtaupp}).
The inner iteration is then a finite-difference method to estimate the gradient of the objective function, utilizing the (biased) pairwise comparison oracles studied in this paper.

\subsection{Regret analysis}

We establish the following theorem upper bounding the regret of Algorithm \ref{alg:acc-pgd}, maximizing strongly concave functions with noisy pairwise comparison oracles.
\begin{theorem}
Suppose $f\in\Sigma_d(2,M)\cap\Gamma_d(\sigma)$ for some $M<\infty$, $\sigma>0$, and Algorithm \ref{alg:acc-pgd} is run with parameters $\eta=\sigma/M$
and $\alpha=1/M$. Then with probability $1-\tilde O(T^{-1})$ the cumulative regret of Algorithm \ref{alg:acc-pgd} is upper bounded by 
$$
 32Md^2 + \left(\frac{2M^2}{\sigma} + \frac{64\sqrt{2}M^5d\sqrt{\gamma_1+2\gamma_2\ln T}}{\sigma^4}\right)\times \sqrt{T}
$$
\label{thm:strongly-concave}
\end{theorem}

\begin{remark}
Dropping polynomial dependency on problem parameters other than the problem dimension $d$ and the time horizon $T$, as well as all poly-logarithmic factors,
the upper bound in Theorem \ref{thm:strongly-concave} is on the asymptotic order of $\tilde O(d^2+d\sqrt{T})$.
\end{remark}

The regret upper bound established in Theorem \ref{thm:strongly-concave} is minimax optimal in terms of dependency on time horizon $T$ up to poly-logarithmic factors,
because it matches the $\Omega(\sqrt{T})$ lower bound proved in \citep{agarwal2010optimal}.


\section{Application to operations management problems}\label{sec:om}

In this section we show how the pairwise comparison oracle discussed in this paper is common in the solution to important operations management questions
with demand learning.
As a clarification, all algorithms and analysis in this section are \emph{not} new;
however, reformulating the existing results into pairwise comparison oracles allows us to use the new algorithms and analysis derived in this paper
to yield new or improved results for these classical learning-while-doing problems in operations management.

\subsection{Example: joint pricing and inventory control with censored demand}\label{subsec:om-inventory}
\subsubsection{Problem formulation.} This example primarily follows \citep[Sec.~EC.6]{chen2023optimal}.
Consider a multi-period joint pricing and inventory replenishment problem for a firm with $\sd$ different types of products.
At the beginning of a time period $t\in[T]$, the firm makes two sets of decisions: the \emph{price decision} $\vct p_t=(p_{t1},\cdots,p_{t\sd})\in[\underline p,\overline p]^{\sd}$
and the \emph{inventory decision} $\vct y_t=\{y_{t1},\cdots,y_{t\sd}\}\in\mathbb R_+^{\sd}$ consisting of inventory order-up-to levels for each product type.
Given $\vct p$, the demands for all products are modeled as $\vct D_t=\vct\lambda(\vct p_t)+\vct\varepsilon_t$ where $\vct\lambda=(\lambda_1(\cdot),\cdots,\lambda_\sd(\cdot))$
are $\sd$ demand curves and $\vct\varepsilon_t=(\varepsilon_{t1},\cdots,\varepsilon_{t\sd})$ are centered noise random variables such that $\vct\varepsilon_1,\cdots,\vct\varepsilon_T\overset{i.i.d.}{\sim}\mu$
for some unknown joint distribution $\mu$ on $\mathbb R^{\sd}$. The firm does \emph{not} know either $\vct\lambda$ or $\mu$.
The instantaneous rewards and inventory level transitions are then the same as in Eq.~(\ref{eq:pricing-inventory-r}) for each product type.

At each time period $t$, the firm observes censored demand $\vct o_t=\max\{\vct y_t-\vct D_t,\vct 0\}$, where the max operator is applied in an element-wise fashion.
With full information of $\mu,\vct\lambda$ and under the asymptotic regime of $T\to\infty$, the near-optimal solution is myopic which solves
\begin{align*}
\vct p^*,\vct y^*=\arg\max_{\vct p,\vct y}&Q(\vct p,\vct y)=\arg\max_{\vct p,\vct y}\sum_{j=1}^{\sd}\left\{p_j\mathbb E_{\mu_j}[\min\{\lambda_j(\vct p)+\varepsilon_j,y_j\}] - b_j\mathbb E_{\mu_j}[(\lambda_j(\vct p)+\varepsilon_j-y_j)^+]\right.\\
&\left.-h_j\mathbb E_{\mu_j}[(y_j-\lambda_j(\vct p)-\varepsilon_j)^+]\right\},
\end{align*}
where $\mu_j$ is the marginal distribution of $\varepsilon_j$ induced by $\mu$ and $(\cdot)^+=\max\{0,\cdot\}$.
The question is then to carry out bandit optimization of $Q(\cdot,\cdot)$ in $2\sd$-dimension with pairwise comparison oracle, to be introduced in the next section.

\subsubsection{Pairwise comparison oracle.}

While $Q(\vct p,\vct y)$ or an unbiased estimate of it cannot be directly accessed due to demand censoring, it is possible to obtain, when given a \emph{pair} of 
price vectors $\vct p,\vct p'\in[\underline p,\overline p]^{\sd}$, an unbiased estimate of 
$$
G(\vct p')-G(\vct p),
$$
where $G(\vct p) := \max_{\vct y\in\mathbb R_+^\sd}Q(\vct p,\vct y)$.
Such a pairwise comparison oracle is constructed in \cite[Algorithm 8]{chen2023optimal}, with the following guarantee summarizing \citep[Lemma EC.6.2]{chen2023optimal}
adapted to the notation of this paper:
\begin{lemma}
Given $\vct p,\vct p'\in[\underline p,\overline p]^{\sd}$ and $n$ time periods, \citep[Algorithm 8]{chen2023optimal} is a pairwise comparison oracle that outputs $\hat\Delta\sim\mathcal O(n,G(\vct p),G(\vct p'))$ that satisfies with probability $1-\tilde O(T^{-2})$ that
$$
\big|\hat\Delta - (G(\vct p')-G(\vct p))\big| \leq O\left(\sd \sqrt{\frac{\ln(\sd nT)}{n}}\right).
$$
\label{lem:pairwise-comparison-inventory}
\end{lemma}
\begin{remark}
In the $O(\cdot)$ notation in Lemma \ref{lem:pairwise-comparison-inventory} we drop dependency on other parameters in regularity conditions.
More detailed dependency terms, including a list of regularity conditions, can be found in \citep{chen2023optimal} and its online supplementary material.
\end{remark}

\subsubsection{Improved regret analysis.} One major improvement, which is a direct consequence of the algorithm and analysis of this paper,
is to incorporate the pairwise comparison oracle in Lemma \ref{lem:pairwise-comparison-inventory} into Algorithm \ref{alg:main-tournament} to deliver an improved regret bound
when the partially optimized objective function $G(\cdot)$ is smoother. More specifically, we have the following result:
\begin{corollary}
Suppose $G\in\Sigma_\sd(k,M)$ for $k\in\mathbb N$, $k\geq 2$. Algorithm \ref{alg:main-tournament} instantiated with pairwise comparison oracle satisfying Lemma \ref{lem:pairwise-comparison-inventory}
has cumulative regret upper bounded by 
$$
O\left((M^2\nu^2+\nu^3+M\nu^2\sqrt{\sd})\times T^{\frac{k+\sd}{2k+\sd}}\ln^5 T\right),
$$
where $\nu=\binom{k+\sd-1}{\sd}$.
\label{cor:om-inventory}
\end{corollary}

Compared with the $\tilde O(T^{\frac{2+\sd}{4+\sd}})$ regret upper bound obtained in \citep{chen2023optimal}, Corollary \ref{cor:om-inventory} improves the existing bound for $k>2$. 
This is because for smoother functions our Algorithm \ref{alg:main-tournament} can better exploit the additional smoothness in contrast to the simplified tournament elimination approach
adopted in \citep{chen2023optimal}.
See also Table \ref{tab:related-works-inventory} for comparisons against additional existing results.

\subsection{Example: network revenue management}\label{subsec:om-nrm}

\subsubsection{Problem formulation.} This example follows \citep{miao2021network} and, to a lesser extent, the setting studied in \citep{chen2019network}.
Consider multi-period pricing of $\sd$ types of products using $\sd'$ types of resources over $T$ time periods with non-replenishable initial inventory levels of the resources
which also scale linearly with $T$. The demand curve $\mathbb E[d_t|p_t]$ which is an $\mathbb R^\sd\to\mathbb R^\sd$ map is unknown, non-parametric and must be learnt
on the fly. Regret is measured against the fluid approximation with the expected revenue as objective and subject to linear feasibility constraints on demand rates.
More specifically, on the demand domain, the fluid approximation is defined as the following constrained optimization problem:
\begin{align}
\lambda^* = \arg\max_{\lambda\in\mathbb R_{++}^\sd} r(\lambda) \;\;\;\;\;\;s.t.\;\; A\lambda\leq \gamma,
\label{eq:nrm-fluid}
\end{align}
where $A\in\mathbb R_{++}^{\sd'\times\sd}$ is the known resource consumption matrix, $\gamma\in\mathbb R_+^{\sd'}$ is a known vector of \emph{normalized}
inventory levels for all resource types, and $r(\lambda)=\langle\lambda, D^{-1}(\lambda)\rangle$ is the expected revenue as a function of the demand rate vector $\lambda$,
where $D^{-1}$ is the inverse function of the demand curve $D(p)=\mathbb E[y|d]$ mapping from $\mathbb R^\sd$ to $\mathbb R^\sd$.

It is a common assumption in the literature that $r$ is both smooth and strongly concave, essentially implying $r\in\Sigma_\sd(2,M)\cap\Sigma_\sd(\sigma)$ for some $M<\infty$ and $\sigma>0$.
This together with the convexity of the feasible region (which is actually a polytope) makes Algorithm \ref{alg:acc-pgd} and Theorem \ref{thm:strongly-concave}
good candidates for obtaining results under this setting of the problem.

\subsubsection{Pairwise comparison oracle.}
While Eq.~(\ref{eq:nrm-fluid}) is a concave maximization problem in $\lambda$, in the bandit setting where $D(\cdot)=\mathbb E[y|\cdot]$ (and therefore $D^{-1}$) is unknown
it is difficult to directly optimize it because it is unknown which price vector $p$ is leading to a target demand rate vector $\lambda$.
This means that if we use optimization methods to obtain a sequence of $\lambda$, it is difficult to either evaluate $r(\lambda)$ or check whether the solution $\lambda$ is feasible
because both involve the unknown demand curve $D$ or its inverse $D^{-1}$.

In \citet{miao2021network} a two-step procedure is proposed which can serve as a pairwise comparison oracle as defined in this paper.
More specifically, the algorithm first uses inverse batch gradient descent type methods similar to the one designed in \citep{chen2019network} to 
find a price that converges to a neighborhood of $D^{-1}(\lambda)$. Afterwards, the algorithm uses Jacobian iteration steps to converge to $D^{-1}(\lambda)$ at a much faster pace.
The following lemma establishes the theoretical property of the two-step procedure in \citep{miao2021network}, which is adapted from \citep[Lemma 3, property 3]{miao2021network}.
\begin{lemma}
Given $\lambda,\lambda'\in\mathbb R_{++}^\sd$ satisfying $A\lambda\leq\gamma$, $A\lambda'\leq\gamma$ and $n$ time periods, \citep[Algorithm 2]{miao2021network} applied to 
$\lambda,\lambda'$ separately is a pairwise-comparison oracle that outputs $\hat\Delta\sim\mathcal O(n,r(\lambda),r(\lambda'))$ that satisfies with probability $1-\tilde O(T^{-2})$ that
\begin{align*}
\big|\hat\Delta-(r(\lambda')-r(\lambda))\big| \leq O\left(\frac{\sd^{2.5}\ln(\sd T)}{\sqrt{n}} + \frac{\sd^5\ln^3(\sd T)}{n}\right),
\end{align*}
where in the $O(\cdot)$ notation we omit polynomial dependency on other problem parameters.
\label{lem:pairwise-comparison-nrm}
\end{lemma}

\subsubsection{Improve regret analysis.} Incorporating the pairwise comparison oracle in Lemma \ref{lem:pairwise-comparison-nrm} into Algorithm \ref{alg:acc-pgd} for $\Sigma_\sd(2,M)\cap \Gamma_\sd(\sigma)$
objective functions, we obtain the following result which improves the result in \citep{miao2021network} without imposing additional assumptions or conditions:
\begin{corollary}
Suppose $r\in\Sigma_\sd(2,M)\cap\Gamma_\sd(\sigma)$. Algorithm \ref{alg:acc-pgd} instantiated with pairwise comparison oracle satisfying Lemma \ref{lem:pairwise-comparison-nrm} has cumulative regret upper bounded by 
$$
O\left(M\sd^2  + \frac{M^5 \sd^{3.5}\ln^{1.5}(\sd T)}{\sigma^4}+ \frac{M^5 \sd^{2.25}\sqrt{ T\ln(\sd T)}}{\sigma^4}\right),
$$
where in the $O(\cdot)$ notation we omit polynomial dependency on other problem parameters.
\label{cor:om-nrm}
\end{corollary}

For large time horizon $T$, the dominating term in Corollary \ref{cor:om-nrm} is $O(\sd^{2.25}\sqrt{T\ln(\sd T)})$, which significantly improves the $O(\sd^{3.5}\sqrt{T}\ln^6(\sd T))$
regret upper bound in \citep[Theorem 3]{miao2021network}.
See also Table \ref{tab:related-works-nrm} for comparisons against additional existing results.

\section{Numerical Results}\label{sec:Exp}

In this section, we provide three different streams of experimental studies to demonstrate the effectiveness of the proposed pairwise comparison algorithms.
We report numerical results on both synthetic objective functions, and also synthetic problem instances arising from inventory management problems.

\subsection{Bandit Optimization with Smooth Objective Functions}\label{subsec: Exp_bandit_opt_smooth}

To mimic the scenario considered in this paper, we conduct two simple bandit optimization problems with smooth objective functions to evaluate the proposed Algorithm 3's finite sample performance. 
We construct the following objective functions:

\begin{equation}
f_1(\mathbf{x}) := 
\begin{cases} 
(1 - \|\mathbf{x}\|_2)^4 (4\|\mathbf{x}\|_2 + 1), & 0 \leq \|\mathbf{x}\|_2 \leq 1, \, \mathbf{x} \in \mathbb{R}^3, \\
0, & \|\mathbf{x}\|_2 > 1.
\end{cases}
\end{equation}


 \[
f_2(\mathbf{x}) :=
\begin{cases} 
(1 - \|\mathbf{x}\|_2)^6 (35\|\mathbf{x}\|_2^2 + 18\|\mathbf{x}\|_2 + 3), & 0 \leq \|\mathbf{x}\|_2 \leq 1, \, x \in \mathbb{R}^3, \\
0, & \|\mathbf{x}\|_2 > 1.
\end{cases}
\]

We follow conventional settings of noises in continuum-armed bandit problems.
That is, for $j\in\{1,2\}$ and every $t\in[T]$, an optimization algorithm takes an action $\mathbf{x}_t\in[0,1]^d$ and observes $z_t'-z_t$ where $z_t=f_j(\mathbf{x}_t)+\epsilon_t$, 
$z_t'=f_j(\mathbf{x}_t')+\epsilon_t'$, with $\epsilon_t,\epsilon_t'\overset{i.i.d.}{\sim}\mathrm{Uniform}([-0.1,0.1])$.
$f_1$ and $f_2$ are the so-called \textit{Wendland functions} that satisfy $f_1\in\Sigma_3(2,M_1)$ and $f_2\in\Sigma_3(4,M_2)$ respectively \citep{chernih2014wendland}.
We then apply Algorithm \ref{alg:main-tournament} to both functions with smoothness levels $k=2,3,4$ for $f_1$ and $k=4,5,6$ for $f_2$.

Throughout our numerical experiments we report the percentage of relative regret, defined as
\begin{equation}
    \frac{[2T\cdot f_j^*-\sum_{t=1}^T(f_j(\mathbf{x}_t)+f_j(\mathbf{x}'_t))]}{2T\cdot f_j^*}\times 100\%,\ j=1,2
\end{equation}
where $f^*_j=\max_{\mathbf{x}\in[0,1]^d}f_j(\mathbf{x})$.
Then, the percentage of relative regret is utilized to evaluate the effectiveness of the proposed pairwise comparison algorithms.
Other parameters for implementing Algorithm 3 are shown in  Appendix \ref{app:parameter_setting}.

\begin{figure}[!t]{}
\centering
{\includegraphics[scale =0.5]{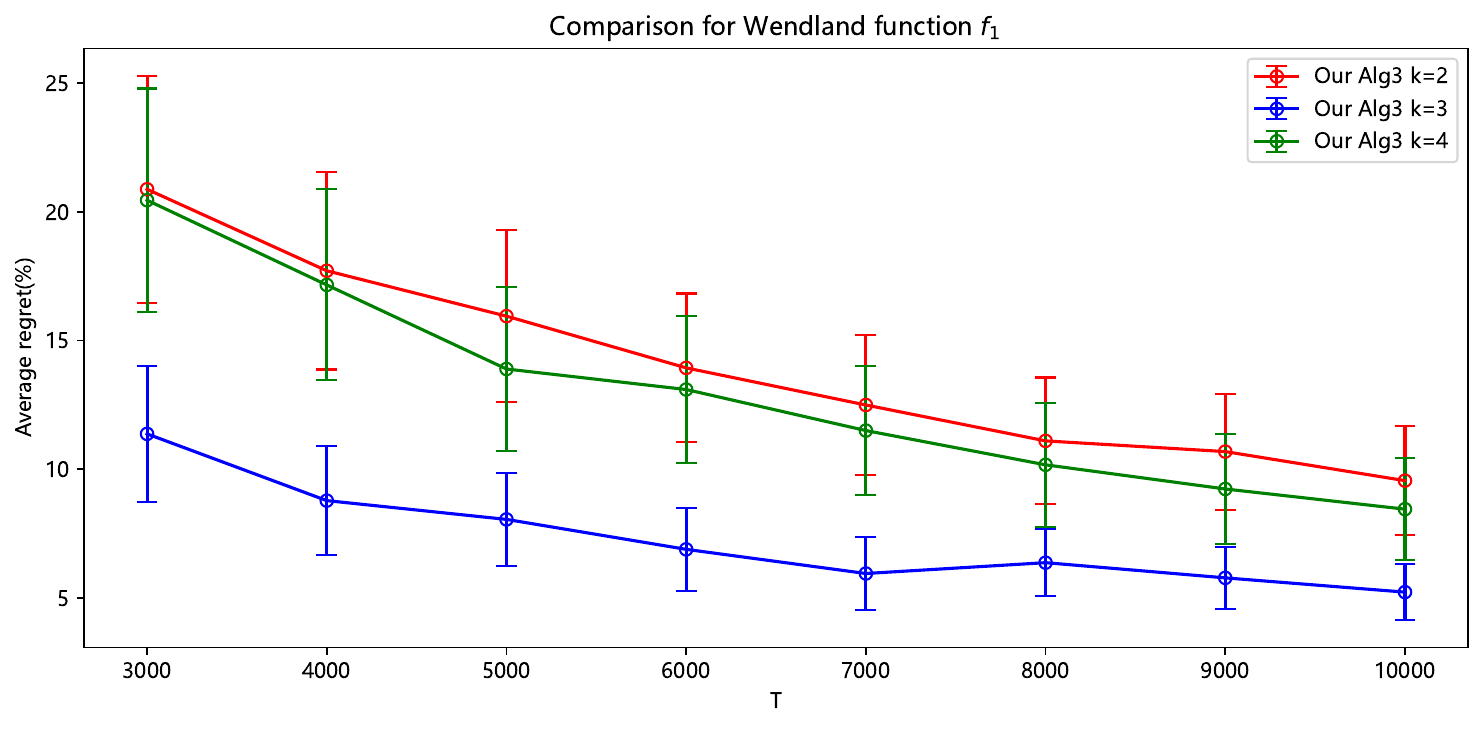}}
\caption{Average regret of $f_1$}\label{fig:f1}
\end{figure}

\begin{figure}[!t]{}
\centering
{\includegraphics[scale =0.5]{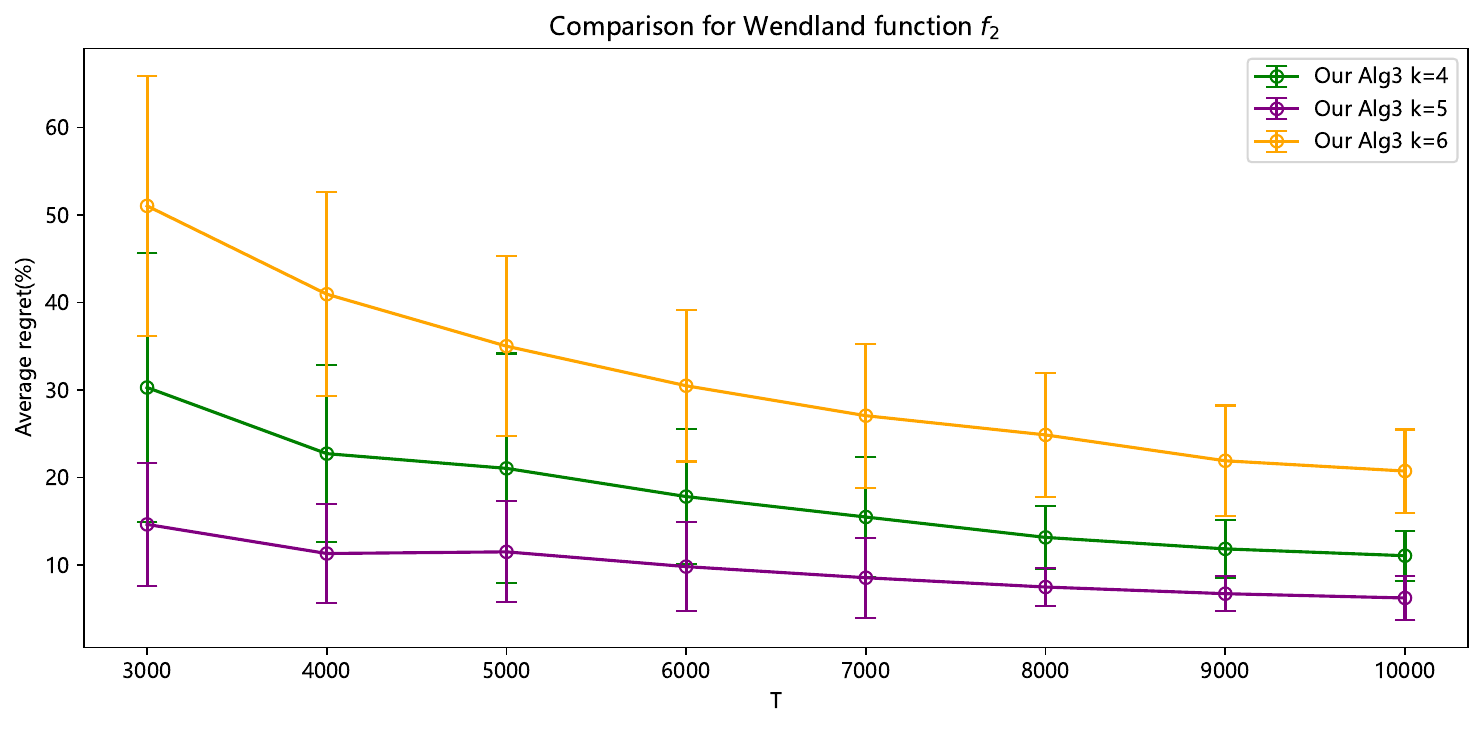}}
\caption{Average regret of $f_2$}\label{fig:f2}
\end{figure}

Figures \ref{fig:f1} and \ref{fig:f2} depict the simulation results of the proposed Algorithm 3 with different smooth levels, with 50 runs, respectively.
From Figure \ref{fig:f1} and \ref{fig:f2}, we can find that as $T\rightarrow\infty$, all the average regrets of Algorithm 3 decrease dramatically.
An interesting result is that the average regret of $f_1$ with the input smooth level $k=3$ is always smaller than others at the same periods.
The reason may be the selected smooth level $k=3$, which is the same as the smoothness of $f_1$.
The same phenomenon can also be observed for the average regret of $f_2$ with the input smooth level $k=5$.
All the above results are consistent with our proposed theorems, which provide more evidence to demonstrate the correctness and effectiveness of the proposed pairwise comparison algorithms.

\subsection{Bandit Optimization with Concave Objective Functions}\label{subsec: Exp_bandit_opt_concave}

Following almost the same setting as the above subsection, we conduct two simple bandit optimization problems with the concave objective functions to compare Algorithm 3 and 4's finite sample performance. 
We consider four objective functions:
\begin{itemize}
    \item $f_3(\mathbf{x})=-\frac{1}{2}\sum_{i=1}^dx_i+1$;
    \item $f_4(\mathbf{x})=-\frac{1}{2}\sum_{i=1}^d(x_i-\frac{1}{4})^2+1$;
\end{itemize}
Obviously, the objective functions $f_j,j=3,4$ are concave and smooth in infinite dimensions.
We set $d=2,3$ and {$k=2, 3, 4$} for this example, and Algorithms 3 and 4 both can apply to the two objective functions.



\begin{figure}[!t]{}
\centering
{\includegraphics[scale =0.45]{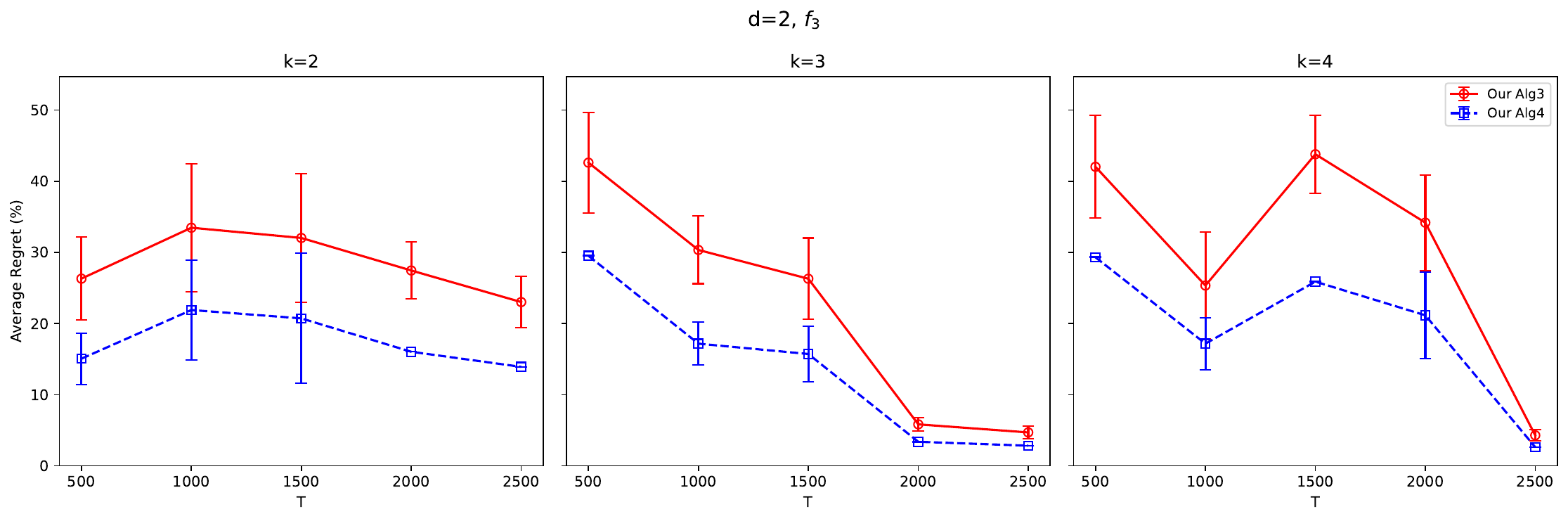}}
\caption{Average regrets of $f_3$ with $d=2$.}\label{fig:d2_f3}
\end{figure}

\begin{figure}[!t]{}
\centering
{\includegraphics[scale =0.45]{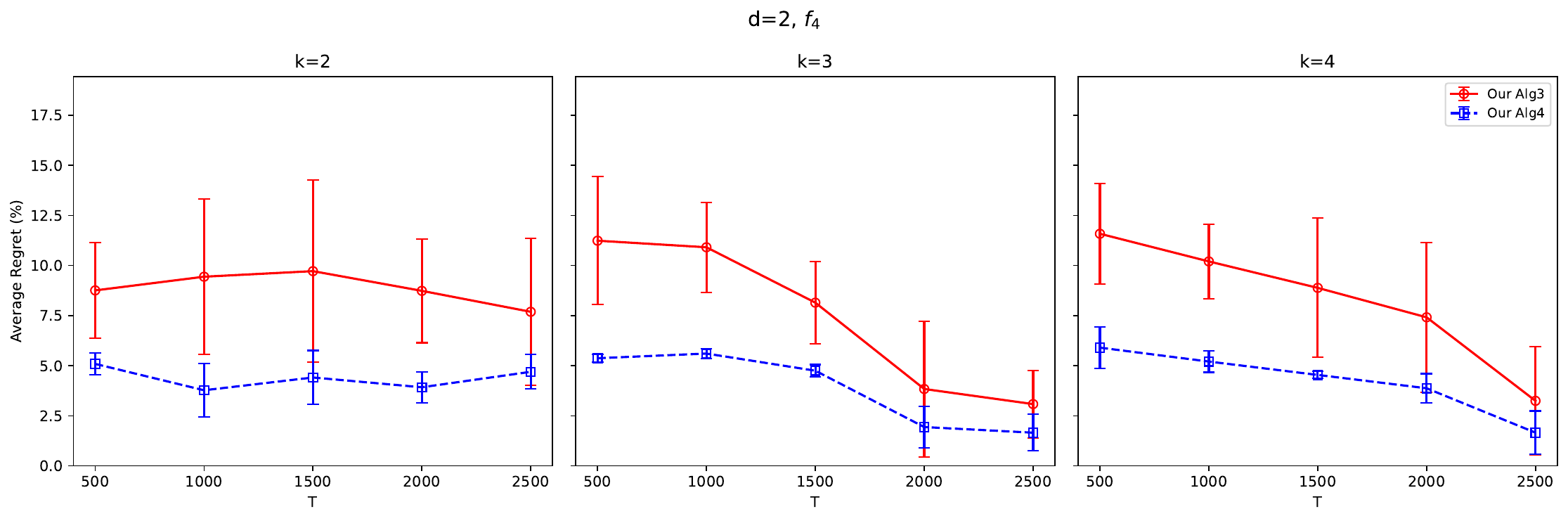}}
\caption{Average regrets of $f_4$ with $d=2$.}\label{fig:d2_f4}
\end{figure}


{Figures \ref{fig:d2_f3} and \ref{fig:d2_f4} show the simulation results of proposed Algorithms 3 and 4 for $f_3$ and $f_4$ with $d=2$, with 50-time runs, respectively.}
The results for $d=3$ can be found in Appendix \ref{app:exp2_results}.
According to the figures, we can find two interesting facts.
First, all the average regrets of Algorithms 3 and 4 decrease dramatically as $T\rightarrow\infty$.
Second, the average regrets of Algorithm 4 are always smaller than those of Algorithm 3 at the same periods and smoothness $k$.
The reason may be that Algorithm 4 is designed for concave objective functions.
Thus, the simulation example is well-suited to the setting of Algorithm 4.

\subsection{Example: Inventory Management with Censored Demand}\label{subsec:Exp_inven}
\subsubsection{Problem formulation}

Let us consider a firm selling one type of product over a time horizon of $T$ periods.
At $1\leq t\leq T$ period, the firm can observe the inventory level $x_t$ before replenishment and need to make a pricing decision $p_t$ and inventory order-up-to decision $y_t\geq x_t$ with zero ordering lead time and zero ordering cost.
Suppose that the unknown demand is influenced by the price $p_t$, namely, $d_t=\lambda(p_t)+\epsilon_t$, where $\lambda(\cdot)$ is a deterministic demand function and $\epsilon_t$ is a noise random variable.
We further assume the \textit{unsatisfied demands are lost and unobservable.}
Thus, the lost-sales quantity $\max\{0,d_t-y_t\}$ is not observable.
At the end of period $t$, the firm incurs a profit of 

\begin{equation}
    r(p_t,y_t) =\underbrace{p_t\min\{d_t,y_t\}}_\text{sales revenue}-\underbrace{b\max\{0,d_t-y_t\}}_\text{lost sales cost}-\underbrace{h\max\{0,y_t-d_t\}}_\text{holding cost},
\end{equation}
where $h$ and $b$ are the per-unit holding and lost-sales penalty costs, respectively.
Note that because the lost-sales cost $b\max\{0,d_t-y_t\}$ is not observable, neither is the realized profit $r(p_t,y_t)$.

It is well known that the optimal pricing and inventory replenishment policy with known $\lambda(\cdot)$ and noise distribution of $\epsilon$ has been investigated by \citet{sobel1981myopic} (called \textit{Clairvoyant optimal policy}), taking the form of 
\begin{equation}
    (p^*,y^*)=\arg\max_{p,y}R(p,y),
\end{equation}
where $R(p,y):=\mathbb{E}_{\epsilon}[r(p,y)]$.
Denote that $y^*(p):=\arg\max_yR(p,y)$ and $G(p):=\max_yR(p,y)=R(p,y^*(p))$.
Therefore, the performance of an admissible policy $\{(p_t,y_t),t\geq 1\}$ is measured by the cumulative regret compared against the Clairvoyant policy $(p^*,y^*)$, i.e.,
\begin{equation}
    \text{regret}=\mathbb{E}\Big[\sum_{t=1}^TR(p^*,y^*)-R(p_t,y_t)\Big].
\end{equation}

In the next part, we test Algorithm 3 under smooth and non-concave $G(\cdot)$ and Algorithm 4 under smooth and concave $G(\cdot)$, respectively, then compare the performance with corresponding algorithms proposed by \citet{chen2023optimal}.
For the noisy pairwise comparison oracle, we adopt the same designed and analyzed in \citep[Algorithm 2, Lemma 2]{chen2023optimal}.
The percentage of relative regret, defined as 
\begin{equation}
    \frac{(T\cdot R(p^*,y^*)-\sum_{t=1}^TR(p_t,y_t))}{T\cdot R(p^*,y^*)}\times 100\%,
\end{equation}
is adopted as the measurement, where $(p^*,y^*)$ is the Clairvoyant optimal reward.

\subsubsection{Settings under different types of $G$}

We test our algorithms using the same settings as in the numerical section of \cite{chen2021nonparametric,chen2023optimal}.
Because the concavity and non-concavity of the objective function $G(\cdot)$ have been conducted by \citet{chen2023optimal}.

To obtain the concave objective function $G(\cdot)$, we set
\begin{itemize}
    \item Cost parameters: $h\in\{2,5\}$ and $b\in\{10,30,50\}$. 
    \item Two demand functions: (i) Exponential function $\lambda(p)=\exp(3-0.04p),p\in[20,30]$; (ii) Logit function $\lambda(p)=\exp(3-0.1p)/(1+\exp(3-0.1p)),p\in[20,30]$.
    \item Random error distribution: (i) Uniform distribution on $[-3,3]$ and normal distribution with mean 0 and standard deviation 2 for the exponential demand function. (ii)Uniform distribution on $[-0.3,0.3]$ and normal distribution with mean 0 and standard deviation 0.2 for the logit demand function.
\end{itemize}

To obtain the non-concave objective function $G(\cdot)$, we set
\begin{itemize}
    \item Cost parameters: $h\in\{0.5,1\}$ and $b\in\{4,6,8\}$. 
    \item Demand functions: $\lambda(p)=5\cdot(2-\Phi(\frac{p-4.5}{0.81})-\Phi(\frac{p-8.5}{1.44})),p\in[1,10]$, where $\Phi(\cdot)$ is the CDF of standard normal distribution.
    
    \item Random error distribution: Uniform distribution on $[-3,3]$ and normal distribution with mean 0 and standard deviation 2 for the exponential demand function. 
\end{itemize}

Combining the cost parameters, demand functions, and error distributions, it shows 24 distinct simulations.
We run the proposed Algorithms 3 and 4 for 50 times and compare the average with the results of corresponding methods in \citet{chen2023optimal}, respectively.
The detailed parameter settings of Algorithms 3 and 4 are given by Appendix \ref{app:parameter_setting}.

\subsubsection{Performance under different types of $G$}

\begin{figure}[t]{}
\centering
{\includegraphics[scale =0.4]{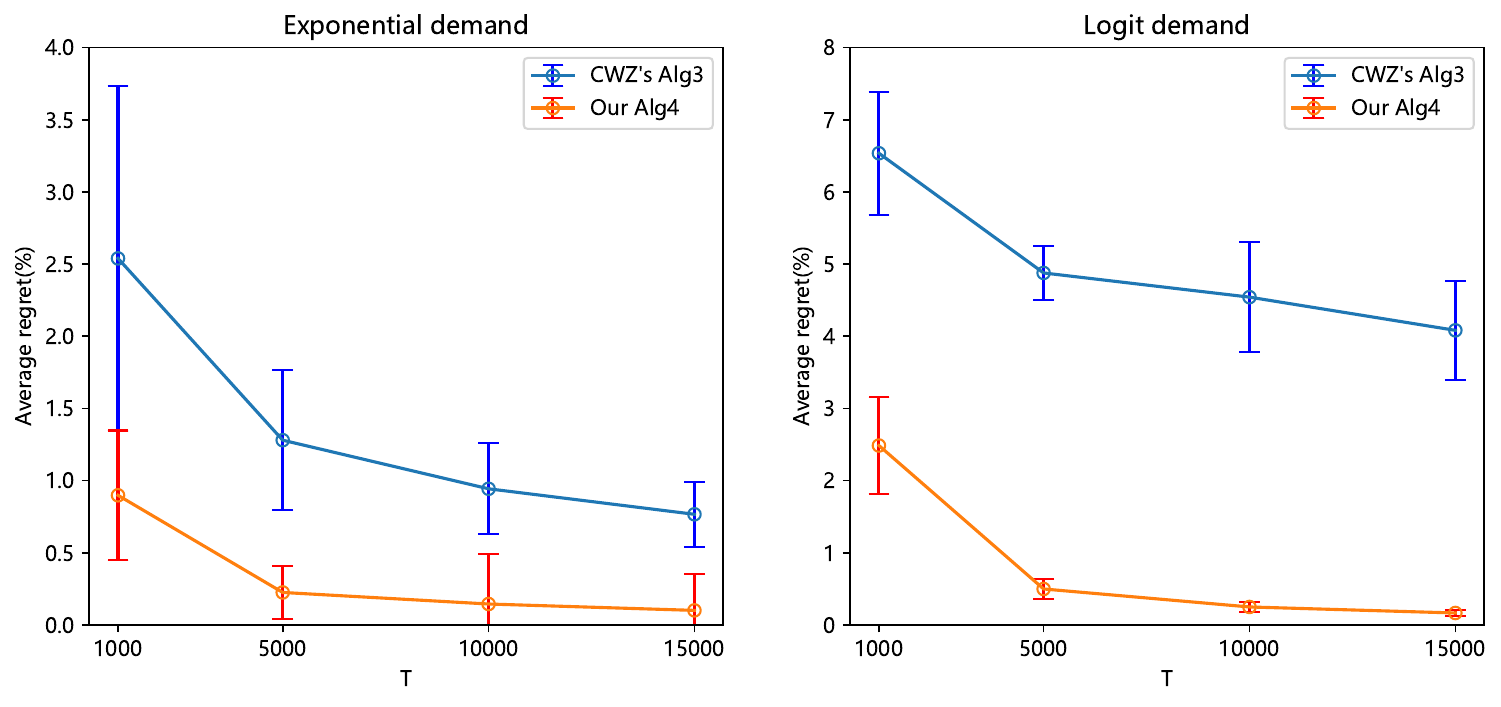}}
{\includegraphics[scale =0.4]{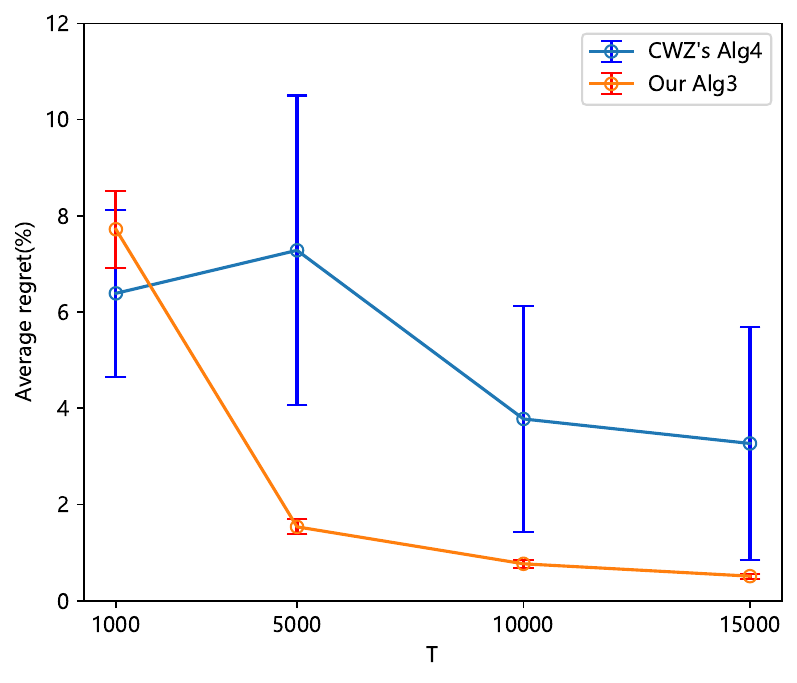}}
\caption{Overall average regret under concave $G(\cdot)$ (left two panels), and non-concave $G(\cdot)$ (the rightmost panel)}\label{fig:concave_inventory}
\end{figure}

Figure~\ref{fig:concave_inventory} reports the results of overall average regrets with the comparison between the proposed Algorithms 3 and 4 with CWZ's Algorithm 4 and 3~\citep{chen2023optimal} for concave and non-concave $G(\cdot)$, respectively.
As it can be seen, Algorithms 3 and 4 significantly outperform CWZ's Algorithms 4 and 3, especially $T>5000$.
The error bars demonstrate that our algorithm is more stable than CWS's.
These results corroborate our theoretical findings that the proposed algorithm is superior on smoother objectives.

\section{Conclusion and future directions}\label{sec:conclusion}

In this paper, we study a biased pairwise comparison oracle abstracted from several important operations management problems.
We proposed bandit algorithms that operate under two different sets of smoothness and/or concavity assumptions,
and derive improved results for two important OM problems (joint pricing and inventory control, network revenue management)
applying our proposed algorithms and their analysis.

To conclude our paper, we mention the following two future directions, one from the bandit optimization perspective and the other one
from applications to OM problems.

\paragraph{Incorporating growth conditions in function classes.} In this paper, we studied two types of function classes: H\"{o}lder smoothness class $\Sigma_d(k,M)$
and strongly concave function class $\Gamma_d(\sigma)$, which is one form of shape constraints.
It would be interesting to extend our results to other important function classes, such as those equipped with growth conditions
that regulate how fast the Lesbesgue measure of level sets of objective functions could grow.
Existing results in continuum-armed bandit show that such growth conditions have significant impacts on the optimal regret of bandit algorithms \citep{locatelli2018adaptivity,wang2019optimization}, and it would be interesting to explore extensions of these results to biased pairwise comparison oracles.

\paragraph{Joint pricing and inventory replenishment with censored demands and fixed costs.} In this paper, we apply our algorithms and analytical frameworks to a joint pricing and inventory replenishment problem with demand learning settings, where demands are censored and inventory ordering is subject to variable, linear costs.
In pracice, however, it is common that inventory ordering is also subject to \emph{fixed costs} that are incurred regardless of ordering amounts.
It is an interesting yet challenging question to extend algorithms and analysis in this paper to problem settings with both demand censoring and inventory ordering costs,
which remains as far as we know an open question in the field.

\bibliographystyle{ormsv080}
\bibliography{refs}

\newpage

\ECSwitch
\numberwithin{lemma}{section}
\numberwithin{proposition}{section}

\ECHead{Supplementary material: additional proofs}

\section{Proofs of results in Sec.~\ref{sec:alg-smooth}}

\subsection{Proof of Lemma \ref{lem:local-poly-approx}}
Because $f\in\Sigma_d(k,M)\subseteq\mathcal C^k([0,1]^d)$, by Taylor expansion with Lagrangian residuals, there exists $\tilde x\in\mathcal X_{\vct j}$
such that
$$
f(x) = \sum_{k_1+\cdots+k_d<k}\frac{1}{(k_1+\cdots+k_d)!}\frac{\partial^{k_1+\cdots+k_d} f(x_{\vct j})}{\partial x_1^{k_1}\cdots\partial x_d^{k_d}}\prod_{i=1}^d (x_i-x_{\vct ji})^{k_i} + \sum_{k_1+\cdots+k_d=k}\frac{1}{k!}\frac{\partial^k f(\tilde x)}{\partial x_1^{k_1}\cdots\partial x_d^{k_d}}\prod_{i=1}^d(x_i-x_{\vct ji})^{k_i}.
$$
With 
$$
\theta_{\vct j} := \left[\frac{1}{(k_1+\cdots+k_d)!}\frac{\partial^{k_1+\cdots+k_d} f(x_{\vct j})}{\partial x_1^{k_1}\cdots\partial x_d^{k_d}}\right]_{k_1+\cdots+k_d<k},
$$
we have that
\begin{align*}
\big|f(x)-\langle\theta_{\vct j},\phi_{\vct j}(x)\rangle\big|
&\leq \sum_{k_1+\cdots+k_d=k}\frac{1}{k!}\left|\frac{\partial^k f(\tilde x)}{\partial x_1^{k_1}\cdots\partial x_d^{k_d}}\prod_{i=1}^d(x_i-x_{\vct ji})^{k_i}\right|
\leq  \sum_{k_1+\cdots+k_d=k}\frac{1}{k!}\times M\times J^{-k}\\
&= \binom{k+d-1}{d}\frac{M}{k!}\times J^{-k} \leq (d+k)^k M\times J^{-k}.
\end{align*}
The $\ell_2$ norm of $\theta_{\vct j}$ can be upper bounded as
\begin{align*}
\|\theta_{\vct j}\|_2 &\leq \sqrt{\nu}\times\max_{k_1+\cdots+k_d<k}\frac{1}{(k_1+\cdots+k_d)!}\left|\frac{\partial^{k_1+\cdots+k_d} f(x_{\vct j})}{\partial x_1^{k_1}\cdots\partial x_d^{k_d}}\right|\leq M\sqrt{\nu}.
\end{align*}
This proves Lemma \ref{lem:local-poly-approx}. $\square$

\subsection{Proof of Lemma \ref{lem:batch-linucb-regret}}
We first prove the second inequality. Let $\Lambda_\tau$ be the $\Lambda$ matrix at the end of iteration $\tau$.
By definition, $\det(\Lambda_1)=\det( I_{\nu\times \nu})=1$ and each iteration the determinant of the matrix doubles.
Hence, $\det(\Lambda_{\tau_0})\geq 2^{\tau_0}$. On the other hand, $\|\Lambda_{\tau_0}\|_\op\leq 1+N\max_\tau\|\phi_\tau\|_2^2\leq
2N\nu$, and therefore $\det(\Lambda_{\tau_0})\leq (2N\nu)^\nu$.
Consequently, $\tau_0\leq \nu\log_2(2N\nu)$.

We next prove the third inequality in Lemma \ref{lem:batch-linucb-regret}.
Let $\hat\theta_\tau\in\mathbb R^\nu$ be the estimated linear model $\hat\theta$ at the beginning of iteration $\tau$.
Let also $\xi_\tau = y_\tau - \langle\theta_{\vct j},\phi_\tau\rangle$.
By definition, we have
\begin{equation}
\sum_{i<\tau}n_i(\langle\theta_{\vct j}-\hat\theta_{\tau},\phi_i\rangle+\xi_i)^2 + \|\hat\theta_\tau\|_2^2\leq \sum_{i<\tau}n_i\xi_i^2 + \|\theta_{\vct j}\|_2^2.
\label{eq:proof-batch-linucb-1}
\end{equation}
Re-arranging terms on both sides of Eq.~(\ref{eq:proof-batch-linucb-1}) and noting that $\Lambda_{\tau-1}=I_{\nu\times\nu}+\sum_{i<\tau} n_i\phi_\tau\phi_\tau^\top$
and $\|\theta_{\vct j}\|_2\leq M\sqrt{\nu}$ thanks to Lemma \ref{lem:local-poly-approx}, we have
\begin{equation}
(\hat\theta_\tau-\theta_{\vct j})^\top\Lambda_{\tau-1}(\hat\theta_{\tau}-\theta_{\vct j})
\leq 2\left|\sum_{i<\tau} n_i\xi_i\langle\phi_i,\hat\theta_{i}-\theta_{\vct j}\rangle\right| + 2M\sqrt{\nu}\|\hat\theta_\tau-\theta_{\vct j}\|_2.
\label{eq:proof-batch-linucb-2}
\end{equation}
Note that $\xi_i=y_i-\langle\theta_{\vct j},\phi_i\rangle$ satisfies $|\xi_i|\leq 2(d+k)^kMJ^{-K}+2\sqrt{(\gamma_1+2\gamma_2\ln T)/n_i}$
with probability $1-O(T^{-2})$, following Definition \ref{defn:oracle} and Lemma \ref{lem:local-poly-approx}.
Subsequently, by the Cauchy-Schwarz inequality, 
\begin{align}
&\left|\sum_{i<\tau} n_i\xi_i\langle\phi_\tau,\hat\theta_{\tau}-\theta_{\vct j}\rangle\right|\leq \sqrt{\tau_0}\sqrt{\sum_{i<\tau}n_i^2\xi_i^2|\langle\phi_i,\hat\theta_\tau-\theta_{\vct j}\rangle|^2}\nonumber\\
&\leq \sqrt{\tau_0}\sqrt{\sum_{i<\tau}n_i^2\left(\frac{4(d+k)^{2k}M^2}{J^{2k}}+\frac{4(\gamma_1+2\gamma_2\ln T)}{n_i}\right)|\langle\phi_i,\hat\theta_\tau-\theta_{\vct j}\rangle|^2}\nonumber\\
&\leq 2\sqrt{((d+k)^{2k}M^2J^{-2k}N+ \gamma_1+2\gamma_2\ln T)\tau_0 (\hat\theta_{\tau}-\theta_{\vct j})^\top\Lambda_{\tau-1}(\hat\theta_{\tau}-\theta_{\vct j})}.
\label{eq:proof-batch-linucb-3}
\end{align}
Here, Eq.~(\ref{eq:proof-batch-linucb-4}) holds because $\sum_i n_i|\langle\phi_i,\hat\theta_\tau-\theta_{\vct j}\rangle|^2 = (\hat\theta_\tau-\theta_{\vct j})^\top(\sum n_i\phi_i\phi_i^\top)(\hat\theta_\tau-\theta_{\vct j}) = (\hat\theta_\tau-\theta_{\vct j}^\top\Lambda_{\tau-1}(\hat\theta-\theta_{\vct j})$ and
$\sum_i n_i^2|\langle\phi_i,\hat\theta_\tau-\theta_{\vct j}\rangle|^2 \leq (\max_i n_i)\times (\sum_i n_i|\langle\phi_i,\hat\theta_\tau-\theta_{\vct j}\rangle|^2)$.
Because $(\hat\theta_\tau-\theta_{\vct j})^\top\Lambda_{\tau-1}(\hat\theta_{\tau}-\theta_{\vct j})\geq \|\hat\theta_\tau-\theta_{\vct j}\|_2^2$
since $\Lambda_{\tau-1}\succeq I$, dividing both sides of Eqs.~(\ref{eq:proof-batch-linucb-2},\ref{eq:proof-batch-linucb-3})
by $\sqrt{(\hat\theta_\tau-\theta_{\vct j})^\top\Lambda_{\tau-1}(\hat\theta_{\tau}-\theta_{\vct j})}$ we obtain
\begin{align}
\sqrt{(\hat\theta_\tau-\theta_{\vct j})^\top\Lambda_{\tau-1}(\hat\theta_{\tau}-\theta_{\vct j})}
&\leq 2M\sqrt{\nu} + 2(d+k)^kMJ^{-k}\sqrt{\tau_0 N} +2 \sqrt{(\gamma_1+\gamma_2N)\tau_0} \nonumber\\
&\leq 2M\sqrt{\nu}+2C_2\sqrt{\tau_{\infty}N} +2\sqrt{C_1\tau_{\infty}} =C_1'.
\label{eq:proof-batch-linucb-4}
\end{align}

Using H\"{o}lder's inequality, it holds for all $x\in\mathcal X_{\vct j}$ that
\begin{align}
\big|\langle\hat\theta_{\tau-1}-\theta_{\vct j}, \phi_{\vct j}(x)-\phi_{\vct j}(\xs)\rangle\big|
&\leq \sqrt{(\hat\theta_\tau-\theta_{\vct j})^\top\Lambda_{\tau-1}(\hat\theta_{\tau}-\theta_{\vct j})}\times \sqrt{(\phi_{\vct j}(x)-\phi_{\vct j}(\xs))^\top\Lambda_{\tau-1}^{-1}(\phi_{\vct j}(x)-\phi_{\vct j}(\xs))}\nonumber\\
&\leq C_1'\sqrt{(\phi_{\vct j}(x)-\phi_{\vct j}(\xs))^\top\Lambda_{\tau-1}^{-1}(\phi_{\vct j}(x)-\phi_{\vct j}(\xs))}.
\label{eq:proof-batch-linucb-5}
\end{align}
On the other hand, Lemma \ref{lem:local-poly-approx} implies that $|f(x)-\langle\theta_{\vct j},\phi_{\vct j}(x)\rangle|\leq (d+k)^kMJ^{-k}$ for all $x\in\mathcal X_{\vct j}$. Consequently, we have for all $x\in\mathcal X_{\vct j}$ that
\begin{align}
\big|[f(x)-f(\xs)] - \langle\theta_{\vct j},\phi_{\vct j}(x)-\phi_{\vct j}(\xs)\rangle\big| &\leq 2(d+k)^kMJ^{-k} + C_1'\sqrt{(\phi_{\vct j}(x)-\phi_{\vct j}(\xs))^\top\Lambda_{\tau-1}^{-1}(\phi_{\vct j}(x)-\phi_{\vct j}(\xs))}\nonumber\\
&\leq C_1'\sqrt{(\phi_{\vct j}(x)-\phi_{\vct j}(\xs))^\top\Lambda_{\tau-1}^{-1}(\phi_{\vct j}(x)-\phi_{\vct j}(\xs))} + C_2.
\label{eq:proof-batch-linucb-6}
\end{align}

Eq.~(\ref{eq:proof-batch-linucb-6}) demonstrates that, with probability $1-\widetilde O(T^{-2})$, the estimate
$\hat f_\tau(x) = \min\{M, \langle\hat\theta, \phi_{\vct j}(x)-\phi_{\vct j}(\xs)\rangle + C_1'\sqrt{(\phi_{\vct j}(x)-\phi_{\vct j}(\xs))^\top\Lambda^{-1}(\phi_{\vct j}(x)-\phi_{\vct j}(\xs))}+C_2\}$ constructed in Line \ref{line:batch-ucb} of Algorithm \ref{alg:batch-linucb-comparison} is a uniform upper bound on $f(\cdot)$,
or more specifically $\hat f_\tau(x)\geq f(x)$ for all $x\in\mathcal X_{\vct j}$.
Subsequently, with probability $1-\tilde O(T^{-2})$ it holds that 
\begin{align}
\sum_{\tau=1}^{\tau_0} n_\tau(f_{\vct j}^*-f(x_\tau))
&\leq  \sum_{\tau=1}^{\tau_0} n_\tau[(\hat f_\tau(x_{\vct j}^*)-\hat f_\tau(x_\tau)) + (\hat f_\tau(x_\tau) - f_\tau(x_\tau)) ]
\leq \sum_{\tau=1}^{\tau_0} n_\tau(\hat f_\tau(x_\tau)-f_\tau(x_\tau))\nonumber\\
&\leq \sum_{\tau=1}^{\tau_0}2n_\tau\min\big\{M, C_1'\sqrt{\phi_\tau^\top\Lambda_{\tau-1}^{-1}\phi_\tau} + C_2\big\}
\leq \sum_{\tau=1}^{\tau_0}2n_\tau(C_2+C_1'\min\big\{M, \sqrt{\phi_\tau^\top\Lambda_{\tau-1}^{-1}\phi_\tau}\big\}) \nonumber\\
&\leq 2C_2N + \sum_{\tau=1}^{\tau_0} 4C_1'\sum_{\ell=0}^{n_\tau-1} \min\{M,\sqrt{2\phi_\tau^\top (\Lambda_{\tau-1}+\ell\phi_\tau\phi_\tau^\top)^{-1}\phi_\tau})\}
\label{eq:proof-batch-linucb-7}\\
&\leq 2C_2 N + 4\sqrt{2}C_1'\sqrt{N}\sqrt{\sum_{\tau=1}^{\tau_0}\sum_{\ell=0}^{n_\tau-1}\min\{M^2,\phi_\tau^\top (\Lambda_{\tau-1}+\ell\phi_\tau\phi_\tau^\top)^{-1}\phi_\tau)\}}\label{eq:proof-batch-linucb-8}\\
&\leq 2C_2 N+8C_1'M\sqrt{N}\sqrt{\sum_{\tau=1}^{\tau_0}\sum_{\ell=0}^{n_\tau-1}\ln(1+\phi_\tau^\top (\Lambda_{\tau-1}+\ell\phi_\tau\phi_\tau^\top)^{-1}\phi_\tau))}
\label{eq:proof-batch-linucb-9}\\
&\leq 2C_2 N+8C_1'M\sqrt{N}\sqrt{\sum_{\tau=1}^{\tau_0}\sum_{\ell=0}^{n_\tau-1}\ln\frac{\det(\Lambda_{\tau-1}+(\ell+1)\phi_\tau\phi_\tau^\top)}{\det(\Lambda_{\tau-1}+\ell\phi_\tau\phi_\tau^\top)}}\nonumber\\
&=2C_2 N+8C_1'M\sqrt{N}\sqrt{\ln\det(\Lambda_{\tau_0})}\nonumber\\
&\leq 2C_2 N+8C_1'M\sqrt{N}\sqrt{\nu\ln(1+N\max_{x\in\mathcal X_{\vct j}}\|\phi_{\vct j}(x)-\phi_{\vct j}(\xs)\|_2^2)}\nonumber\\
&\leq 2C_2N + 8C_1'M\sqrt{\nu N\ln(5\nu N)}.\label{eq:proof-batch-linucb-10}
\end{align}
Here, 
Eq.~(\ref{eq:proof-batch-linucb-7}) holds because $\det(\Lambda_{\tau-1})\leq \det(\Lambda_{\tau-1}+\ell\phi_\tau\phi_\tau^\top)\leq 2\det(\Lambda_{\tau-1})$
by the definition of $n_\tau$, and hence $\phi_\tau^\top\Lambda_{\tau-1}^{-1}\phi_\tau\leq 2\phi_\tau^\top(\Lambda_{\tau-1}+\ell\phi_\tau\phi_\tau^\top)^{-1}\phi_\tau$
thanks to \citep[Lemma 4]{abbasi2011improved};
Eq.~(\ref{eq:proof-batch-linucb-8}) holds by the Cauchy-Schwarz inequality and the fact that $\sum_{\tau=1}^{\tau_0}n_\tau=N$;
Eq.~(\ref{eq:proof-batch-linucb-9}) holds because for any $M\geq 1$ and $x>0$, $\min\{M,x\}=M\min\{1,x/M\} \leq 2M\ln(1+x/M)\leq 2M\ln(1+x)$.

Finally, the first inequality in Lemma \ref{lem:batch-linucb-regret} holds because $n_{\tau^*}\geq N/\tau_0$, and therefore
$$
f_{\vct j}^*-f(x_{\tau^*}) \leq \frac{N}{n_{\tau^*}}\sum_{\tau=1}^{\tau_0} n_\tau(f_{\vct j}^*-f(x_\tau)) \leq (2C_2N + 8C_1'M\sqrt{\nu N\ln(5\nu N)})\frac{\nu\log_2(2\nu N)}{N},
$$
which is to be proved. $\square$

\subsection{Proof of Lemma \ref{lem:iterative-batch-linucb-regret}}
The first property holds by directly invoking Lemma \ref{lem:batch-linucb-regret} and noting that $2^{\beta_0}\geq N/2$ and the definition of $C_1'$.

To prove the second property, note that for any iteration $\beta\leq\beta_0$ of length $N_\beta=2^{\beta}$, 
Algorithm \ref{alg:iterative-batch-lin-ucb} invokes the \textsc{BatchLinUCB} with $\xs=x_{\beta-1}$.
The first and third properties of Lemma \ref{lem:batch-linucb-regret} show that, with probability $1-\tilde O(T^{-2})$, 
the total regret cumulated in the $N_\beta$ time periods (accounting for both $f(\xs)$ and $f(x_\tau)$ in Algorithm \ref{alg:batch-linucb-comparison}) is upper bounded by
\begin{align*}
2C_2N_\beta&+8C_1'M\sqrt{\nu N_\beta\ln(5\nu N_\beta)} + N_\beta\left(3C_2 \nu \ln(2\nu N_\beta) + 12C_1'M\frac{\sqrt{\nu^3\ln^2(5\nu N_\beta)}}{\sqrt{N_{\beta-1}}}\right) \\
&\leq 6C_2\nu N_\beta\ln(2\nu N) + 17C_1' M\sqrt{\nu^3 N_\beta\ln^2(5\nu N)}.
\end{align*}
Subsequently, 
\begin{align*}
\sum_{\beta=1}^{\beta_0}&6C_2\nu N_\beta\ln(2\nu N) + 17C_1' M\sqrt{\nu^3 N_\beta\ln^2(5\nu N)}\\
&\leq 6C_2\nu N\ln(2\nu N) + 21C_1'M\sqrt{\nu^3 N\ln^3(5\nu N)}\\
&\leq 56C_2 M\nu^2 N\ln^2(5\nu N) + 42M\nu^2\sqrt{C_1 N}\ln^2(5\nu N),
\end{align*}
which is to be proved. $\square$

\subsection{Proof of Lemma \ref{lem:tournament-keys}.}
We use induction to prove the first property.
The inductive hypothesis $\vct j^*\in\mathcal A_\zeta$ for $\zeta=1$ is trivially true because $\mathcal A_1=[J]$.
Now assume $\vct j^*\in\mathcal A_\zeta$ holds, we shall prove that $\vct j^*\in\mathcal A_{\zeta+1}$ holds as well with probability $1-\tilde O(T^{-2})$.

Fix arbitrary $\vct j,\vct j'\in\mathcal A_\zeta$ and let $y=\mathcal O(N_\zeta, x_{\vct j,\zeta},x_{\vct j,\zeta'})$.
By Lemma \ref{lem:iterative-batch-linucb-regret} and Definition \ref{defn:oracle}, we have with probability $1-\tilde O(T^{-2})$ that
\begin{align}
&\big|y-(f_{\vct j'}^*-f_{\vct j}^*)\big|
\leq \big|y-(f(x_{\vct j',\zeta})-f(x_{\vct j,\zeta}))\big| + \big|f_{\vct j}^*-f(x_{\vct j,\zeta})\big| + \big|f_{\vct j'}^*-f(x_{\vct j',\zeta})\big|\nonumber\\
&\leq \frac{\sqrt{\gamma_1+2\gamma_2\ln T}}{\sqrt{N_\zeta}} + \frac{2(34M+42\sqrt{C_1})M\nu^2\sqrt{\ln^3(5\nu T)}}{\sqrt{N_\zeta}} + 2(6\nu+58M)C_2\nu^2 \ln^{1.5}(5\nu T)\nonumber\\
&\leq \frac{\sqrt{C_3}}{\sqrt{N_\zeta}} + {C_2'} \leq {\varepsilon_\zeta+C_2'}.\label{eq:proof-tournament-keys-1}
\end{align}
The elimination steps (Line \ref{line:step-elimination}) then satisfy the following properties: with Eq.~(\ref{eq:proof-tournament-keys-2}),
the $y\gets \mathcal O(N_\zeta,x_{\vct j^*,\zeta},x_{\hat{\vct j}(\zeta),\zeta})$ feedback satisfies with probability $1-\tilde O(T^{-2})$ that
\begin{align*}
y \leq f_{\hat{\vct j}(\zeta)}^*-f_{\vct j^*}^* + {\varepsilon_\zeta+C_2'}\leq\varepsilon_\zeta+C_2',
\end{align*}
where the second inequality holds because $f_{\hat{\vct j}(\zeta)}^*\leq f_{\vct j^*}^*$ thanks to the definition of $\vct j^*$.
This shows that $\vct j^*\in\mathcal A_{\zeta+1}$, which is to be proved.

We next prove the second property of Lemma \ref{lem:tournament-keys}.
Because $\omega\leq \log_2 J+1\leq 1.45\ln(2J)\leq 1.45\ln(2T)=:\bar\omega$ almost surely, Eq.~(\ref{eq:proof-tournament-keys-1}) and the tournament mechanism in Algorithm 
\ref{alg:main-tournament} imply that
\begin{equation}
f_{\vct j^*}^* - f_{\hat{\vct j}(\zeta)}^* \leq \bar\omega(\varepsilon_\zeta + C_2').
\label{eq:proof-tournament-keys-2}
\end{equation}
On the other hand, according to Line \ref{line:step-elimination} of Algorithm \ref{alg:main-tournament}, for any $\vct j\in\mathcal A_{\zeta+1}$,
the $y\gets\mathcal O(N_\zeta,x_{\vct j,\zeta}, x_{\hat{\vct j}(\zeta),\zeta})$ satisfies $y\leq\varepsilon_\zeta+C_2'$.
Subsequently, applying Eqs.~(\ref{eq:proof-tournament-keys-1},\ref{eq:proof-tournament-keys-2}), we obtain
\begin{align}
f_{\vct j^*}^*-f_{\vct j}^* 
&\leq [f_{\vct j^*}^*-f_{\hat{\vct j}(\zeta)}^*] + \big|y-(f_{\hat{\vct j}(\zeta)}^*-f_{\vct j}^*)\big| + y\leq 3\bar\omega(\varepsilon_\zeta+C_2').
\end{align}
This completes the proof of Lemma \ref{lem:tournament-keys}. $\square$

\subsection{Proof of Theorem \ref{thm:regret-smooth-functions}}
To simplify notations, throughout this proof we shall drop numerical constants and polynomial dependency on $d,k$.
Let $\zeta_{\infty}$ be the last complete iteration in Algorithm \ref{alg:main-tournament} and fix arbitrary $\zeta\leq\zeta_{\infty}$.
The cumulative regret incurred by all operations during iteration $\zeta$ can be upper bounded, with probability $1-\tilde O(T^{-2})$, by
\begin{align}
\underbrace{\bar\omega N_\zeta\left[\sum_{\vct j\in\mathcal A_\zeta}f_{\vct j^*}^*-f(x_{\vct j,\zeta})\right]}_{=:\mathfrak A(\zeta), \text{from Step \ref{line:step-tournament}}}
+ \underbrace{N_\zeta\left[\sum_{\vct j\in\mathcal A_\zeta}2f_{\vct j^*}^* - f(x_{\vct j,\zeta})-f(x_{\hat{\vct j}(\zeta),\zeta})\right]}_{=:\mathfrak B(\zeta), \text{from Step \ref{line:step-elimination}}}
+ \underbrace{\left[\sum_{\vct j\in\mathcal A_\zeta}\sum_t 2f_{\vct j^*}^*-f(x_t)-f(x_t')\right]}_{=:\mathfrak C(\zeta), \text{from Step \ref{line:step-preprocess}}}
\label{eq:proof-thm-smooth-1}
\end{align}

Using Lemmas \ref{lem:iterative-batch-linucb-regret} and \ref{lem:tournament-keys}, and a union bound over all $|\mathcal A_\zeta|\leq J^d\leq T$ hypercubes,
it holds with probability $1-\tilde O(T^{-1})$ that
\begin{align}
\mathfrak A(\zeta) &\leq \bar\omega N_\zeta\times \sum_{\vct j\in\mathcal A_\zeta} \big\{[f_{\vct j^*}^*-f_{\vct j}^* ] + [f_{\vct j}^*-f(x_{\vct j,\zeta})]\big\}\nonumber\\
&\leq \bar\omega N_\zeta\times \sum_{\vct j\in\mathcal A_\zeta}\left[3\bar\omega(\varepsilon_\zeta+C_2') + O\left((\sqrt{C_1}+M)\nu^2\sqrt{\frac{\ln^3 T}{N_\zeta}}+(M+\nu)C_2\nu^2\ln^{1.5} T\right)\right]\nonumber\\
&\leq O(((M+\sqrt{C_1})\nu^2\ln^{1.5}T + \overline\omega\sqrt{C_3})\times \bar\omega|\mathcal A_\zeta|\sqrt{N_\zeta} + ((M+\nu)C_2\ln^{1.5}T+\overline\omega C_2')\times \bar\omega|\mathcal A_\zeta|N_\zeta)\nonumber\\
&\leq O(((M+\sqrt{C_1})\nu^2\ln^{1.5}T + \overline\omega\sqrt{C_3})\times \bar\omega \sqrt{J^dN_\zeta^\tot} + ((M+\nu)C_2\ln^{1.5}T+\overline\omega C_2')\times \bar\omega N_\zeta^\tot).
\label{eq:proof-thm-smooth-2}
\end{align}

Using the same analysis, $\mathfrak B(\zeta)$ can also be upper bounded by 
\begin{equation}
\mathfrak B(\zeta)\leq O(((M+\sqrt{C_1})\nu^2\ln^{1.5}T + \overline\omega\sqrt{C_3})\times \bar\omega \sqrt{J^dN_\zeta^\tot} + ((M+\nu)C_2\ln^{1.5}T+\overline\omega C_2')\times \bar\omega N_\zeta^\tot).
\label{eq:proof-thm-smooth-3}
\end{equation}

Using Lemma \ref{lem:iterative-batch-linucb-regret} and a union bound over all $|\mathcal A_\zeta|\leq J^d\leq T$ hypercubes,
it holds with probability $1-\tilde O(T^{-1})$ that
\begin{align}
\mathfrak C(\zeta) &\leq O(C_2M\nu^2|\mathcal A_\zeta|N_\zeta\ln^2 T + M\nu^2|\mathcal A_\zeta|\sqrt{C_1 N_\zeta}\ln^2 T)
\leq O(M\nu^2(C_2N_\zeta^\tot + \sqrt{C_1 J^d N_\zeta^\tot})\ln^2 T).
\label{eq:proof-thm-smooth-4}
\end{align}

Define the following constants:
\begin{align*}
K_1 &:= (M+\sqrt{C_1})\nu^2\bar\omega\ln^{1.5}T + \bar\omega^2\sqrt{C_3}\leq O((M+\sqrt{\gamma_1+\gamma_2\ln T})M\nu^2\ln^{3.5}T);\\
K_2 &:= C_2 M\nu^2\bar\omega\ln^2 T +C_2'\bar\omega^2\leq O((M^2\nu^2+\nu^3)J^{-k}\ln^{3.5}T).  
\end{align*}
Summing over $\zeta=1,2,\cdots,\zeta_{\infty}$ and noting that $\sum_{\zeta=1}^{\zeta_{\infty}}N_\zeta^\tot\leq T$, we have that
\begin{align}
&\sum_{\zeta=1}^{\zeta_{\infty}} \mathfrak A(\zeta)+\mathfrak B(\zeta)+\mathfrak C(\zeta)
\leq \sum_{\zeta=1}^{\zeta_{\infty}}O(K_1\sqrt{J^d N_\zeta^\tot} + K_2 N_\zeta^\tot)\leq O(K_1\sqrt{\zeta_{\infty} J^dT} + K_2 T).
\label{eq:proof-thm-smooth-5}
\end{align}

Finally, note that $N_\zeta^\tot\geq N_\zeta \geq 1/\varepsilon_\zeta^{2} \geq 4^{-\zeta}$ and therefore $\zeta_{\infty}\leq 0.5\log_2 T$
because $\sum_\zeta N_\zeta^\tot\leq T$.
Expanding all definitions of constants in Eq.~(\ref{eq:proof-thm-smooth-5}), the regret of Algorithm \ref{alg:main-tournament} is upper bounded by
\begin{align*}
 O(K_1\sqrt{J^dT\ln T} + K_2 T) &\leq (M^2\nu^2+\nu^3+M\nu^2\sqrt{\gamma_1+\gamma_2\ln T})\times O((\sqrt{J^dT}+J^{-k}T)\ln^4 T)\\
&\leq (M^2\nu^2+\nu^3+M\nu^2\sqrt{\gamma_1+\gamma_2\ln T})\times O(T^{\frac{k+d}{2k+d}}\ln^4 T),
\end{align*}
which is to be proved. Note that in this algorithm the implemented solutions $\{x_t\}_{t=1}^T$ are always feasible; that is, $x_t\in\mathcal Z$ for all $t$.
Therefore, the regret of the $\psi(\cdot)$ term is exactly zero. $\square$

\section{Proof of Theorem \ref{thm:strongly-concave}}

The first part of the proof analyzes the estimation error of the gradient estimate $\hat{\vct g}_\tau$ as well as the regret incurred in the gradient estimation procedure, by establishing the following lemma:
\begin{lemma}
For every $\tau$, let $x_{\tau,j}'$ and $x_{\tau,j}''$ be the solutions involved on Lines \ref{line:xtaup} and \ref{line:xtaupp} corresponding to dimension $j\in[d]$.
 Then with probability $1-\tilde O(T^{-1})$ the following hold:
\begin{enumerate}
\item $\|-\hat{\vct g}_\tau-\nabla \tilde f(x_\tau)\|_2\leq \sqrt{d}(Mh_\tau+h_\tau^{-1}\sqrt{(\gamma_1+\gamma_2\ln(NT))/\beta_\tau})$, where $\tilde f(x) = -f(x) - \sigma\|x\|_2^2/2$;
\item $\sum_{j=1}^d \beta_\tau[2f(x_\tau)-f(x_{\tau,j}')-f(x_{\tau,j}'')] \leq \beta_\tau dMh_\tau^2$.
\end{enumerate}
\label{lem:est-g}
\end{lemma}
\begin{proof}{Proof of Lemma \ref{lem:est-g}.}
We first prove the second property of Lemma \ref{lem:est-g}. By Taylor expansion and the fact that $f\in\Sigma_d(2,M)$, it holds for every $j\in[d]$ that
\begin{align}
\big|f(x_{\tau,j}')-f(x_\tau)-h_\tau\partial_j f(x_\tau)\big| \leq Mh_\tau^2;\label{eq:proof-estg-1}\\
\big|f(x_{\tau,j}'')-f(x_\tau)+h_\tau\partial_j f(x_\tau)\big| \leq Mh_\tau^2.\label{eq:proof-estg-2}
\end{align}
Combining Eqs.~(\ref{eq:proof-estg-1},\ref{eq:proof-estg-2}) we proved the second property. 

We next turn to the first property. Eqs.~(\ref{eq:proof-estg-1},\ref{eq:proof-estg-2}) together yield that
\begin{equation}
\left|\frac{f(x_{\tau,j}')-f(x_{\tau,j}'')}{2h_\tau} - \partial_j f(x_\tau)\right| \leq Mh_\tau^2, \;\;\;\;\;\;\forall j\in[d]
\label{eq:proof-estg-3}
\end{equation}
On the other hand, the definition of the noisy pairwise comparison oracle $\mathcal O$ combined together with the union bound yield that, with probability $1-\tilde O(T^{-1})$, the following
hold for all $j\in[d]$:
\begin{align}
\big|y_\tau(j) - (f(x_{\tau,j}')-f(x_\tau))\big| &\leq \sqrt{\frac{\gamma_1+\gamma_2\ln(NT)}{\beta_\tau}};\label{eq:proof-estg-4}\\
\big|y_\tau(j)'-(f(x_{\tau,j}'')-f(x_\tau))\big| &\leq \sqrt{\frac{\gamma_1+\gamma_2\ln(NT)}{\beta_\tau}}.\label{eq:proof-estg-5}
\end{align}
Combining Eqs.~(\ref{eq:proof-estg-3},\ref{eq:proof-estg-4},\ref{eq:proof-estg-5}) and noting that $\nabla\tilde f(x)=-\nabla f(x)-\sigma x$, we complete the proof of Lemma \ref{lem:est-g}. $\square$
\end{proof}

Our next lemma establishes convergence rate of the proximal gradient descent/ascent approach adopted in Algorithm \ref{alg:acc-pgd},
using fixed step sizes and \emph{inexact} gradient estimates $\hat{\vct g}_\tau$.
\begin{lemma}
Let $\tilde f(x)=-f(x)-\sigma\|x\|_2^2/2$, which is convex on $\mathcal X$ thanks to the condition that $f\in\Gamma_d(\sigma)$.
For each epoch $\tau$, let $\varepsilon_\tau := \|-\hat{\vct g}_\tau-\nabla\tilde f(x_\tau)\|_2$. Then for every $\tau_0\geq 1$ until $T$ time periods are elapsed,
$$
\sum_{\tau=1}^{\tau_0} 4d\beta_\tau [f^*-f(x_\tau)] \leq 4d(1+M/\sigma)\left(\sqrt{\frac{\sigma}{2}}\|x_0-x^*\|_2 + \sqrt{\frac{2}{\sigma}}\sum_{\tau=1}^{\tau_0}\sqrt{\beta_\tau}\varepsilon_\tau\right)^2,
$$
where $f^*=f(x^*)=\max_{x\in\mathcal Z}f(x)$.
\label{lem:inexact-gd}
\end{lemma}

The technical proof of Lemma \ref{lem:inexact-gd} is involved and therefore deferred to the end of this section.

We are now ready to complete the proof of Theorem \ref{thm:strongly-concave}. The first property of Lemma \ref{lem:est-g} implies that, with high probability, 
$\varepsilon_\tau \leq \sqrt{d}(Mh_\tau+h_\tau^{-1}\sqrt{(\gamma_1+\gamma_2\ln(NT))/\beta_\tau})\leq \sqrt{d}(Mh_\tau+h_\tau^{-1}\sqrt{(\gamma_1+2\gamma_2\ln T)/\beta_\tau}$,
where the second inequality holds because $N\leq T$ almost surely. Incorporating this into Lemma \ref{lem:inexact-gd},
we have with probability $1-\tilde O(T^{-1})$ that
\begin{align}
\sum_{\tau=1}^{\tau_0}& 4d\beta_\tau [f^*-f(x_\tau)] 
\leq 8Md\|x_0-x^*\|_2^2 + \frac{32Md}{\sigma^{2}}\left(\sum_{\tau=1}^{\tau_0}M\sqrt{\beta_\tau d}h_\tau + \frac{\sqrt{\gamma_1+2\gamma_2\ln T}}{h_\tau}\right)^2\nonumber\\
&\leq 8Md\|x_0-x^*\|_2^2 + \frac{128M^3 \sqrt{\gamma_1+2\gamma_2\ln T}}{\sigma^2} \left(\sum_{\tau=1}^{\tau_0}\sqrt[4]{\beta_\tau d}\right)^2\nonumber\\
&= 8Md\|x_0-x^*\|_2^2 + \frac{128M^3 d^{3/2} \sqrt{\gamma_1+2\gamma_2\ln T}}{\sigma^2}\left(\sum_{\tau=1}^{\tau_0}(1+\eta)^{\tau/4}\right)^2\nonumber\\
&\leq 8Md\|x_0-x^*\|_2^2 + \frac{128M^3 d^{3/2} \sqrt{\gamma_1+2\gamma_2\ln T}}{\sigma^2}\times \frac{(1+\eta)^{(\tau_0+1)/2}}{\eta^2}\nonumber\\
&\leq 8Md\|x_0-x^*\|_2^2 + \frac{128M^3 d^{3/2} \sqrt{\gamma_1+2\gamma_2\ln T}}{\sigma^2}\times \frac{\sqrt{1+\eta}}{\eta^2}\times \sqrt{\beta_{\tau_0}}\nonumber\\
&\leq 8Md\|x_0-x^*\|_2^2 + \frac{128M^3 d^{3/2} \sqrt{\gamma_1+2\gamma_2\ln T}}{\sigma^2}\times \frac{\sqrt{1+\eta}}{\eta^2}\times \sqrt{\frac{T}{4d}},\label{eq:proof-strongly-concave-2}
\end{align}
where the last inequality holds because a complete epoch $\tau$ lasts for $4d\beta_\tau$ periods, and therefore $4d\beta_{\tau_0}\leq T$.
Because $\eta=\sigma/M\leq 1$ and $x_0,x^*\in\mathcal X\subseteq[0,1]^d$ which implies $\|x_0-x^*\|_2\leq 2\sqrt{d}$, Eq.~(\ref{eq:proof-strongly-concave-2}) yields
\begin{align}
\sum_{\tau=1}^{\tau_0}& 4d\beta_\tau [f^*-f(x_\tau)] \leq 32Md^2 + \frac{64\sqrt{2}M^5d \sqrt{\gamma_1+2\gamma_2\ln T}}{\sigma^4}\times \sqrt{T}.
\label{eq:proof-strongly-concave-3}
\end{align}
Combining Eq.~(\ref{eq:proof-strongly-concave-3}) with the second property of Lemma \ref{lem:est-g}, and noting that the regret incurred by the $\psi(\cdot)$ term is zero
because the violation constraints from $x_\tau'$ and $x_\tau''$ are symmetric and cancel out each other, we would complete the proof of Theorem \ref{thm:strongly-concave}.
More specifically, the regret arising from the second property of Lemma \ref{lem:est-g} can be upper bounded by
\begin{align}
\sum_{\tau=1}^{\tau_0}Md\beta_\tau h_\tau^2 &\leq \sum_{\tau=1}^{\tau_0} M\sqrt{\beta_\tau d} \leq M\sqrt{d}\left(\sum_{\tau=1}^{\tau_0}(1+\eta)^{\tau/2}\right)
\leq M\sqrt{d}\times \frac{(1+\eta)^{(\tau_0+1)/2}}{\sqrt{1+\eta}-1}\nonumber\\
&\leq 4M\sqrt{d}\times \frac{\sqrt{\beta_{\tau_0}}}{\eta} \leq \frac{4M\sqrt{d}}{\eta}\sqrt{\frac{T}{4d}} \leq \frac{2M^2}{\sigma}\times \sqrt{T}.
\end{align}

\subsection{Proof of Lemma \ref{lem:inexact-gd}}

By Line \ref{line:iter-pgd} of Algorithm \ref{alg:acc-pgd}, it holds for every epoch $\tau$ that
\begin{equation}
\langle-\hat{\vct g}_\tau, x-x_{\tau+1}\rangle + \frac{\sigma\|x\|_2^2}{2} - \frac{\sigma\|x_{\tau+1}\|_2^2}{2} + \frac{1}{2\alpha_\tau}\left(\|x-x_\tau\|_2^2 -\|x_{\tau+1}-x_\tau\|_2^2\right) \geq \frac{\alpha^{-1}+\sigma}{2}\|x-x_{\tau+1}\|_2^2.
\label{eq:proof-inexact-gd-1}
\end{equation}
Because $\tilde f$ is continuously differentiable with $M$-Lipschitz continuous gradients, it holds that
\begin{align}
\tilde f(x_{\tau+1}) &\leq \tilde f(x_\tau) + \langle\nabla\tilde f(x_\tau),x_{\tau+1}-x_\tau\rangle + \frac{M}{2}\|x_{\tau+1}-x_\tau\|_2^2.
\label{eq:proof-inexact-gd-2}
\end{align}
Additionally, using the definition that $\varepsilon_\tau = \|-\hat{\vct g}_\tau-\nabla\tilde f(x_\tau)\|_2$ and the Cauchy-Schwarz inequality, 
we have for every $x\in\mathcal Z$ that
\begin{align}
\langle -\hat{\vct g}_\tau, x-x_{\tau+1}\rangle + \langle-\hat{\vct g}_\tau, x_{\tau+1}-x_\tau\rangle 
&\leq \tilde f(x)-\tilde f(x_\tau)+\varepsilon_\tau\|x-x_{\tau+1}\|_2.
\label{eq:proof-inexact-gd-3}
\end{align}
Combining Eqs.~(\ref{eq:proof-inexact-gd-1},\ref{eq:proof-inexact-gd-2},\ref{eq:proof-inexact-gd-3}) and using the fact that $\alpha=1/M$, we have for every $x\in\mathcal Z$ that
\begin{align}
\frac{\sigma+\alpha^{-1}}{2}\|x-x_{\tau+1}\|_2^2 -f(x_{\tau+1})
&\leq -f(x) + \frac{M-\alpha^{-1}}{2}\|x_{\tau+1}-x_\tau\|_2^2 + \frac{1}{2\alpha}\|x-x_\tau\|_2^2 + \varepsilon_\tau\|x-x_{\tau+1}\|_2\nonumber\\
&\leq -f(x)+ \frac{1}{2\alpha}\|x-x_\tau\|_2^2 + \varepsilon_\tau\|x-x_{\tau+1}\|_2.\label{eq:proof-inexact-gd-4}
\end{align}

Let $c_\tau = (1+M/\sigma)^\tau = (1+\alpha\sigma)^\tau$ be a coefficient associated with epoch $\tau$, so that $c_0=1$ and $c_\tau(\sigma+\alpha^{-1})=c_{\tau+1}\alpha^{-1}$.
By Eq.~(\ref{eq:proof-inexact-gd-4}) setting $x=x^*= \arg\max_{x\in\mathcal Z}f(x)$, we have that
\begin{align}
\frac{c_{\tau+1}}{2\alpha}\|x^*-x_{\tau+1}\|_2^2 + c_\tau(f(x^*)-f(x_{\tau+1})) \leq \frac{c_\tau}{2\alpha}\|x^*-x_\tau\|_2^2 + c_\tau\varepsilon_\tau\|x_{\tau+1}-x^*\|_2.\label{eq:proof-inexact-gd-5}
\end{align}
Now let $\delta_\tau := f(x^*)-f(x_\tau)$ be the sub-optimality gap at the solution of epoch $\tau$. Summing both sides of Eq.~(\ref{eq:proof-inexact-gd-5}) through $\tau=0,\cdots,t-1$ and telescoping,
we obtain for every $t\leq\tau_0$ that
\begin{align}
\frac{c_{t}}{2\alpha}\|x_t-x^*\|_2^2 \leq \frac{c_t}{2\alpha}\|x_t-x^*\|_2^2 + \sum_{i=1}^t c_{i-1}\delta_i \leq \frac{c_0}{2\alpha}\|x^*-x_0\|_2^2 + \sum_{i=1}^t c_{i-1}\varepsilon_{i-1}\|x_i-x^*\|_2.
\label{eq:proof-inexact-gd-6}
\end{align}

We next present and prove a technical lemma that is helpful for our recursive analysis:
\begin{lemma}\label{lem:seq_sqbound}
Let $\{v_t\}_{t\ge 0}$ and $\{b_t\}_{t\ge 1}$ be nonnegative real sequences. If $v_t^2 \le a + \sum_{i=1}^t b_iv_i$ for all $t\ge 1$ and $v_0^2\le a$, 
then for any $t\ge 0$, we have 
\begin{equation*}
{a + \sum_{i=1}^t b_iv_i} \le (\sqrt{a} + {\sum}_{i=1}^t b_i)^2. 
\end{equation*}
\end{lemma}
\begin{proof}{Proof of Lemma \ref{lem:seq_sqbound}.}
Indeed, let $A_0:= a$ and $A_t := a + \sum_{i=1}^t b_iv_i$ for all $t\ge 1$, so that $v_t\le \sqrt{A_t}$ for all $t\ge 0$ and $\{A_t\}_{t\ge 0}$ is nonnegative and non-decreasing. Then we have for all $t\ge 0$,
\begin{equation*}
A_{t+1} = A_t + b_{t+1}v_{t+1} \le A_t + b_{t+1}\sqrt{A_{t+1}},
\end{equation*}
 and hence
\begin{equation}
\sqrt{A_{t+1}} \le {A_t}/{\sqrt{A_{t+1}}} + b_{t+1} \le \sqrt{A_t} + b_{t+1}. \label{eq:A_recur}
\end{equation}
Summing~\eqref{eq:A_recur} over $i=0,\ldots,t-1$ and we have
\begin{equation*}
\sqrt{A_{t}}\le \sqrt{A_0} + \sum_{i=1}^t b_{i} =  \sqrt{a} + \sum_{i=1}^t b_{i},
\end{equation*}
which is to be proved. $\square$
\end{proof}

Invoke Lemma \ref{lem:seq_sqbound} with $v_t=\sqrt{c_t\alpha^{-1}/2}\|x_t-x^*\|_2$, $a=c_0\alpha^{-1}\|x_0-x^*\|_2^2/2 = \alpha^{-1}\|x_0-x^*\|_2^2/2$ and $b_i=\sqrt{2}c_{i-1}\varepsilon_{i-1}\sqrt{c_i\alpha^{-1}}$ for $i=1,2,\cdots,t$; we obtain
\begin{align}
\sum_{i=1}^t c_{i-1}\delta_i &\leq \left(\sqrt{\alpha^{-1}/2}\|x_0-x^*\|_2 + \sum_{i=1}^t \frac{\sqrt{2\alpha}c_{i-1}\varepsilon_{i-1}}{\sqrt{c_i}}\right)^2.
\label{eq:proof-inexact-gd-7}
\end{align}
Note that for each epoch $\tau$, $c_\tau=(1+M/\sigma)^\tau=(1+\alpha\sigma)^\tau=\beta_\tau = (1+M/\sigma)\beta_{\tau-1}$.
Eq.~(\ref{eq:proof-inexact-gd-7}) then implies the conclusion in Lemma \ref{lem:inexact-gd}.

\section{Parameter setting of Experiment one}

\subsection{Setting for bandit optimization with smooth objective functions}\label{subsec:setting_bandit_exp1}

The parameter setting for $f_1, f_2$ is in Table \ref{tab:bandit_opt_exp1_f1} and Table \ref{tab:bandit_opt_exp1_f2}, respectively. 
Only the parameters with updated values compared to Table~\ref{tab:bandit_opt_exp1_f1} are reported in Table \ref{tab:bandit_opt_exp1_f2}.

\begin{table}[!ht]
\centering
\caption{Parameters and their values for bandit optimization for $f_1$}\label{tab:bandit_opt_exp1_f1}
\begin{threeparttable}
\renewcommand{\arraystretch}{1.15}
\begin{tabular}{|c|l|c|}
\hline
\textbf{Parameter Symbol} & \textbf{Explanation}                & \textbf{Value} \\ \hline
$C_1$  &  parameter in Algorithm \ref{alg:batch-linucb-comparison}        & $\gamma_1 + 2 \gamma_2\ln T$    \\ \hline
$C_1^{\prime}$  & parameter in Algorithm \ref{alg:batch-linucb-comparison}         & $\sqrt{\nu}\ln T+C_2\sqrt{\tau_\infty N}+\sqrt{C_1\tau_\infty}$    \\ \hline
$\tau_{\infty}$  &  parameter in $C_1^{\prime}$         & $\nu\log_{2}(2N\nu)$    \\ \hline
$\nu$  & parameter in upper bound in Lemma \ref{lem:local-poly-approx}         & $\binom{k+d-1}d$    \\ \hline
$\gamma_1$  & parameter in $C_1$         & 0.01    \\ \hline
$\gamma_2$  & parameter in $C_1$         & 0.005    \\ \hline
$C_2$  &  parameter in Algorithm \ref{alg:main-tournament}        & $10^{-4} (d+k)^kJ^{-k}\ln T$    \\ \hline
$C_2^{\prime}$  &   parameter in Algorithm \ref{alg:main-tournament}       & $(12\nu+116M)C_2\nu^2\ln^{1.5}(5\nu T)$    \\ \hline
$C_3$  &  parameter in Algorithm \ref{alg:main-tournament}        & $(68M+85\sqrt{C_1})^2M^2\nu^4\ln^3(5\nu T)$    \\ \hline
$M$  &  control parameter of H\"{o}lder class        & $10^{-4}$    \\ \hline
$k$ & smoothness level         & $[2, 3, 4]$    \\ \hline
$J$          & cube count per dimension   & 3            \\ \hline
$d$        & dimensions used for analysis        & $3$ \\ \hline
\makecell{$T$} & \makecell[l]{total time steps} & 3000-10000(step=1000) \\ \hline
$\epsilon$ & noise in the oracle function         & $0.1$    \\ \hline
\end{tabular}
\begin{tablenotes}
\footnotesize
\item[*] In practical implementations, we suggest scaling $C_1, C_1',C_2,C_2',C_3$ by constants.
\item[**] Since we use $k_1 + ... + k_d < k$ in the definition of $\phi_{\vct j}(x)$, a second-degree polynomial is actually used for fitting by setting $k=3$.
\end{tablenotes}
\end{threeparttable}
\label{tab:parameters}
\end{table}

\begin{table}[!ht]
\centering
\caption{Parameters and their values for bandit optimization for $f_2$}\label{tab:bandit_opt_exp1_f2}
\begin{threeparttable}
\renewcommand{\arraystretch}{1.15}
\begin{tabular}{|c|l|c|}
\hline
\textbf{Parameter Symbol} & \textbf{Explanation}                & \textbf{Value} \\ \hline
$M$  &  control parameter of H\"{o}lder class        & $\frac{1}{\sqrt{\nu}}$    \\ \hline
$k$ & smoothness level         & $[4, 5, 6]$    \\ \hline
$\epsilon$ & noise in the oracle function        & $0.3$    \\ \hline
\end{tabular}
\end{threeparttable}
\label{tab:parameters}
\end{table}

\section{Parameter setting and results of Experiment two}\label{app:exp2}

\subsection{Setting for bandit optimization with concave objective functions}\label{subsec:setting_bandit}

The parameters used in Algorithm \ref{alg:main-tournament} and Algorithm \ref{alg:acc-pgd} for concave objective functions are demonstrated in Table \ref{tab:bandit_opt}.

\begin{table}[h]
\centering
\caption{Parameters and their values for bandit optimization with concave objective functions}\label{tab:bandit_opt}
\begin{threeparttable}
\renewcommand{\arraystretch}{1.15}
\begin{tabular}{|c|l|c|}
\hline
\textbf{Parameter Symbol} & \textbf{Explanation}                & \textbf{Value} \\ \hline
$C_1$  &  parameter in Algorithm \ref{alg:batch-linucb-comparison}        & $\gamma_1 + 2 \gamma_2\ln T$    \\ \hline
$C_1^{\prime}$  & parameter in Algorithm \ref{alg:batch-linucb-comparison}         & $\sqrt{\nu}\ln T+C_2\sqrt{\tau_\infty N}+\sqrt{C_1\tau_\infty}$    \\ \hline
$\tau_{\infty}$  &  parameter in $C_1^{\prime}$         & $\nu\log_{2}(2N\nu)$    \\ \hline
$\nu$  & parameter in upper bound in Lemma \ref{lem:local-poly-approx}         & $\binom{k+d-1}d$    \\ \hline
$\gamma_1$  & parameter in $C_1$         & 0.01    \\ \hline
$\gamma_2$  & parameter in $C_1$         & 0.01    \\ \hline
$C_2$  &  parameter in Algorithm \ref{alg:main-tournament}        & $(d+k)^kJ^{-k}\ln T$    \\ \hline
$C_2^{\prime}$  &   parameter in Algorithm \ref{alg:main-tournament}       & $(12\nu+116M)C_2\nu^2\ln^{1.5}(5\nu T)$    \\ \hline
$C_3$  &  parameter in Algorithm \ref{alg:main-tournament}        & $(68M+85\sqrt{C_1})^2M^2\nu^4\ln^3(5\nu T)$    \\ \hline
$M$  &  control parameter of H\"{o}lder class        & $10^{-4}$    \\ \hline
$k$ & smoothness level         & $[2, 3, 4]$    \\ \hline
$J$          & cube count per dimension   & 3            \\ \hline
$d$        & dimensions used for analysis        & $2, 3$ \\ \hline
$T$        & total time steps        & $[100, 500, 1000, 1500, 2000, 2500]$ \\ \hline
$\eta$    & parameter in Algorithm \ref{alg:acc-pgd}        & 1 \\ \hline
$\sigma$        & parameter in Algorithm \ref{alg:acc-pgd}  &  1  \\ \hline
$\alpha$        & parameter in Algorithm \ref{alg:acc-pgd}  &  10 \\ \hline
\end{tabular}
\begin{tablenotes}
\footnotesize
\item[*] In practical implementations, we suggest scaling $C_1, C_1',C_2,C_2',C_3$ by constants.
\item[**] Since we use $k_1 + ... + k_d < k$ in the definition of $\phi_{\vct j}(x)$, a second-degree polynomial is actually used for fitting by setting $k=3$.
\end{tablenotes}
\end{threeparttable}
\label{tab:parameters}
\end{table}

\subsection{Results of Experiment two}\label{app:exp2_results}

\begin{figure}[!t]{}
\centering
{\includegraphics[scale =0.45]{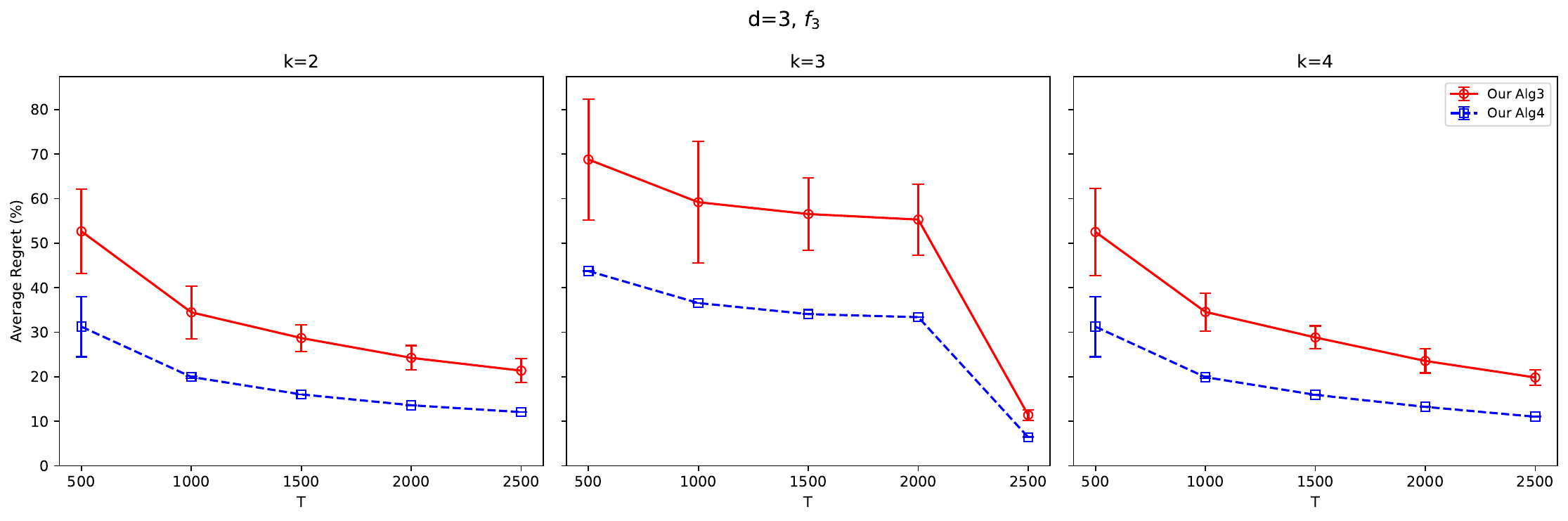}}
\caption{Average regrets of $f_3$ with $d=3$.}\label{fig:d3_f3}
\end{figure}

\begin{figure}[!t]{}
\centering
{\includegraphics[scale =0.45]{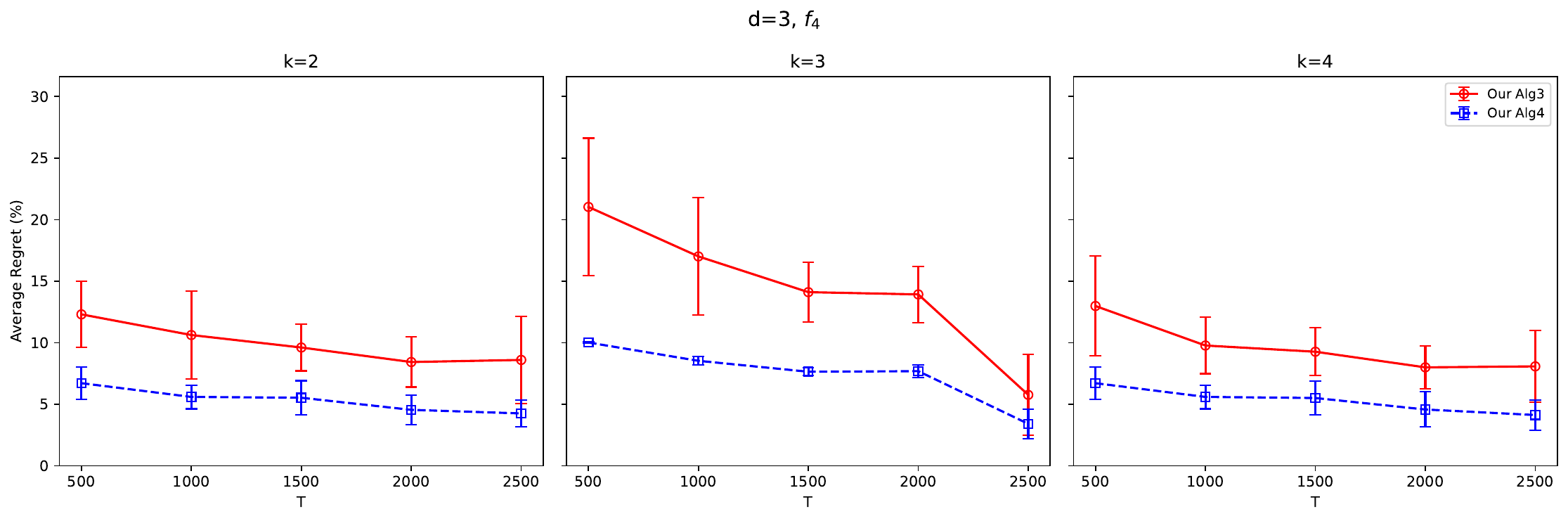}}
\caption{Average regrets of $f_4$ with $d=3$.}\label{fig:d3_f4}
\end{figure}

 \section{Parameter Setting of Experimental Examples}\label{app:parameter_setting}

\end{document}